\documentclass{article}
\usepackage[preprint]{log_2026}			% for preprint version

\usepackage{booktabs}						% professional-quality tables
\usepackage{multirow}						% tabular cells spanning multiple rows
\usepackage{amsfonts}						% blackboard math symbols
\usepackage{graphicx}						% figures

\usepackage{amsthm}

\usepackage{algorithm}
\usepackage{algpseudocode}

\usepackage{adjustbox}
\usepackage{subcaption}
\usepackage[numbers,compress,sort]{natbib}	% for numerical citations

\usepackage{bbm}

\usepackage{mathtools}
\newtheorem{theorem}{Theorem}
\newtheorem{corollary}{Corollary}
\newtheorem{lemma}{Lemma}
\newtheorem{proposition}{Proposition}
\newtheorem{definition}{Definition}

\newcommand{\autograble}{\textsc{AutoGrable}\xspace} 
\newcommand{\C}{\mathcal{C}}          \newcommand{\val}{\mathrm{val}}

\newcommand{\Ggamma}[1]{G^{#1}_{C,T}}
  \newcommand{\ginc}{\gamma_{\mathrm{inc}}}

\newcommand{\SCS}{\textsc{SCS}\ }

\title[AutoGrable: What Is a Good Graph for a Table?]{AutoGrable: What Is a Good Graph for a Table?}

\author{%
Tamara Cucumides \\
University of Antwerp \\
\email{tamara.cucumidesfaundez@uantwerp.be}\And
Floris Geerts\\
University of Antwerp \\
\email{floris.geerts@uantwerp.be}
}

\begin{document}

\maketitle

\begin{abstract}
Graph learning presupposes a graph, and tables and relational databases do not come with one. Applying a GNN to them requires deciding which entities become nodes, which of them to connect, and through which relations---a decision made by hand, by schema heuristics, or by training a model on every candidate graph and keeping the best. We give a criterion that requires no trained graph model. In the minimal table-to-graph abstraction each row is a node, so a message-passing GNN, bounded by 1-WL, sees a construction only as a partition of the rows into colour-refinement classes: a construction is good for a task when that partition separates rows with different labels and does not split rows that share one. \autograble turns this criterion into a construction procedure. For incidence constructions the partition is fixed by the selected columns, so building a graph reduces to choosing them, and we score a candidate subset by a label-alignment risk: the held-out risk of the best predictor constant on its blocks, penalised by an occupancy term measuring how thinly the blocks are populated. The score materialises no graph and trains no GNN, so \autograble can search the space of subsets greedily and cheaply, and returns the resulting \emph{grable} for single tables and for foreign-key schemas alike. Our experiments show that over a space of candidate graphs the score discards a large fraction while retaining the best; that \autograble recovers the columns that generate the label on controlled tasks and outperforms fixed, random, and task-aware constructors on real tasks under a fixed predictor; and that it is the only method compared that can decline to build a graph
when none helps.
\end{abstract}

\section{Introduction}
Graph learning presupposes the existence of a single or multiple graphs. The graph is where the inductive bias lives: Message-passing can only relate two entities if an edge, or a path of edges, connects them. Recent work has applied graph learning techniques to the data formats that dominate applied machine learning-- single tables and relational databases linked by foreign keys-- but these do not come with a graph representation. In order to bring a graph model to such data we must first decide what the nodes are, which entities to connect, and through which relations. This decision is usually made by hand, schema heuristics, or not at all, and it is made \emph{before} any learning happens. 

The question this paper takes up is the one that precedes architecture and optimization entirely: \textbf{when the graph is not given, what is a good graph to learn on, and how do we build it?}

Two responses to this question exist, yet neither answers it fully. One \emph{learns} a graph: taking the data points as nodes, it infers a soft adjacency from their features jointly with the model~\citep{franceschi, kazi2022dgm, zhou2023opengsl}. The other \emph{constructs} one from a table or schema, fixing the foreign-key skeleton, searching schema edits, or scoring attributes, and judges it by the accuracy of a model trained on it~\citep{RDL-fey, autog, augraph}. The first assumes structure means geometric proximity over given nodes, which a table does not provide: its relation is shared categorical values, and its node set is itself a construction choice. The second defines a good graph only operationally, good if a trained model likes it, which is circular, costs a training run per candidate, and says nothing about \emph{which distinctions a graph exposes to the learner}. We seek a notion of a good graph that can be stated before training and that explains why a construction helps.

In this work, we judge the quality of a graph representation by the distinctions the model running on it can make. Message-passing GNNs are bounded by the one-dimensional Weisfeiler--Leman (1-WL) test, or colour refinement~\citep{xu2018how,morris-WL-go-neural}: a GNN cannot tell two nodes apart when 1-WL assigns them the same colour, that is, when their neighborhoods are indistinguishable under iterated aggregation, however near or far apart the two nodes sit. For such a model, then, a graph constructed from a table does just one thing: it sorts the table rows (row nodes in the graph) into 1-WL colour classes, and only distinctions between colour classes are visible to the learner. A graph is good for a task when this sorting lines up with the labels, rows of different labels fall in different classes, and rows of the same label are not split further than necessary. Too coarse a sorting hides the labels; too fine a one lets the model memorize individual rows instead of generalizing. This gives a criterion for constructing a graph that can be stated, and checked, before any model is trained.
 
We turn this target into an algorithm. Scoring a candidate graph is cheap: because the target only concerns which rows are grouped together, it can be measured directly from that grouping, with no model training. Yet, searching for the best graph is not cheap, the closely related separation problem is NP-hard, so we build the graph greedily rather than by enumeration. The result is \textsc{AutoGrable}, which selects a construction by this training-free target and then trains a graph model on it. The algorithm naturally extends from a single table to several tables linked by foreign keys.

\textbf{Contributions.}
We propose a criterion for what makes a graph good for a task, and apply it to tabular learning, where the graph is not given but has to be constructed. Our contributions are (1) \emph{alignment} as a criterion of graph goodness: a construction reaches a 1-WL-bounded learner only as a grouping of the rows, and is good for a task when that grouping matches the labels; (2) a score for this criterion, computed from the grouping alone with no model trained, weighing the error of the best predictor the grouping admits against how thinly it spreads the rows; and (3) \autograble, a table-to-graph constructor that builds the graph the score selects, applies to single- and multi-table datasets, and
outperforms alternative constructions on transactional and relational
benchmarks.
% \autograble\ recovers planted generating columns on controlled tasks and, under a fixed predictor, beats fixed, random and task-aware constructors on transactional, relational and i.i.d.\ benchmarks, including by building no graph when none helps.
\section{Attribute selection as label alignment}
\label{sec:preliminaries}
Our approach rests on one observation: selecting a subset of attributes and expanding its values into shared nodes of an incidence graph induces, on the rows of the table, exactly the colour-refinement (1-WL) partition of that graph -- and colour refinement is precisely what limits the power of message-passing GNNs (Section~\ref{sec:autograble}). Choosing attributes therefore fixes the expressive ceiling of every GNN run on the constructed graph, and the choice can be made entirely in table space, even though what the choice produces, cross-row communication, is not available in table space at all. 

We thus focus first on how to select the attribute subset. The guiding notion is that its induced partition should be \emph{aligned} with the labels: fine enough that cells are label-homogeneous, yet coarse enough that each cell retains the support needed to estimate anything on it. This section makes alignment precise. We attach to every candidate attribute set a partition of the training rows and a training-free predictor, and score the pair by an objective that trades held-out predictive gain against fragmentation, with a generalisation bound as justification.

\textbf{Tables and candidate attributes.}
Let $\C$ be a universe of attribute names and $\val$ a value domain. A
\emph{schema} is a finite set $C\subseteq\C$, a \emph{$C$-row} is a map
$r\colon C\to\val$, and a \emph{$C$-table} is a finite indexed collection of
rows. We write $r[c]$ for the value of $c$ in $r$, and $r|_S$ for the
restriction of $r$ to $S\subseteq C$. One attribute $Y\in C$ is the label,
$A:=C\setminus\{Y\}$ are the input attributes, and $F\subseteq A$, fixed
before the validation sample is inspected, contains the attributes eligible
for structural selection. The labelled rows are split into disjoint training
and validation samples $T_{\mathrm{tr}}$ and $T_{\mathrm{val}}$ of sizes
$n_{\mathrm{tr}}$ and $n_{\mathrm{val}}$.
Every candidate set $S\subseteq F$ partitions the training rows by
projection:
\begin{equation*}
    r\sim_S r'
    \quad\Longleftrightarrow\quad
    r|_S=r'|_S,
    \qquad
    \pi_S:=T_{\mathrm{tr}}/{\sim_S}.
    \label{eq:selected-attribute-equivalence}
\end{equation*}
Equivalently, $\pi_S$ consists of the nonempty cells
$B_{S,u}:=\{r\in T_{\mathrm{tr}}:r|_S=u\}$ for $u\in\val^S$, with
$N_{S,u}:=|B_{S,u}|$. We write $\pi'\preceq\pi$ when $\pi'$ refines $\pi$;
in particular $S\subseteq S'$ implies $\pi_{S'}\preceq\pi_S$.

\textbf{A training-free block predictor.}
Each candidate set is scored with the simplest predictor compatible with its
partition. Writing
$\widehat p_S(y\mid u):=N_{S,u}^{-1}\sum_{r\in T_{\mathrm{tr}}}
\mathbf 1\{r|_S=u,\ r[Y]=y\}$ and
$\widehat p_0(y):=n_{\mathrm{tr}}^{-1}\sum_{r\in T_{\mathrm{tr}}}
\mathbf 1\{r[Y]=y\}$, we set
$\widehat h_S(r):=\widehat p_S(\,\cdot\mid r|_S)$ if $N_{S,r|_S}>0$ and
$\widehat h_S(r):=\widehat p_0$ otherwise: the empirical label distribution on
occupied cells, the training marginal on unseen projections. This predictor
serves only to select a structure; the final predictor is a graph-learning
model (Section~\ref{sec:autograble}).

\textbf{The alignment score $\mathcal J$.}
Alignment must balance two failure modes. An under-refined partition mixes
differently labelled rows in one cell, which no cell-constant predictor can
resolve; an over-refined one captures more label variation but produces
small, poorly supported cells. For a subset $S\subset F$ of attributes, we measure fragmentation by
\begin{equation*}
    \Omega(T_{\mathrm{tr}},\pi_S)
    :=
    \frac{1}{n_{\mathrm{tr}}}
    \sum_{u:N_{S,u}>0}\sqrt{N_{S,u}},
    \label{eq:main-occupancy}
\end{equation*}
\begin{figure}[t]
    \centering
    \includegraphics[width=\linewidth]{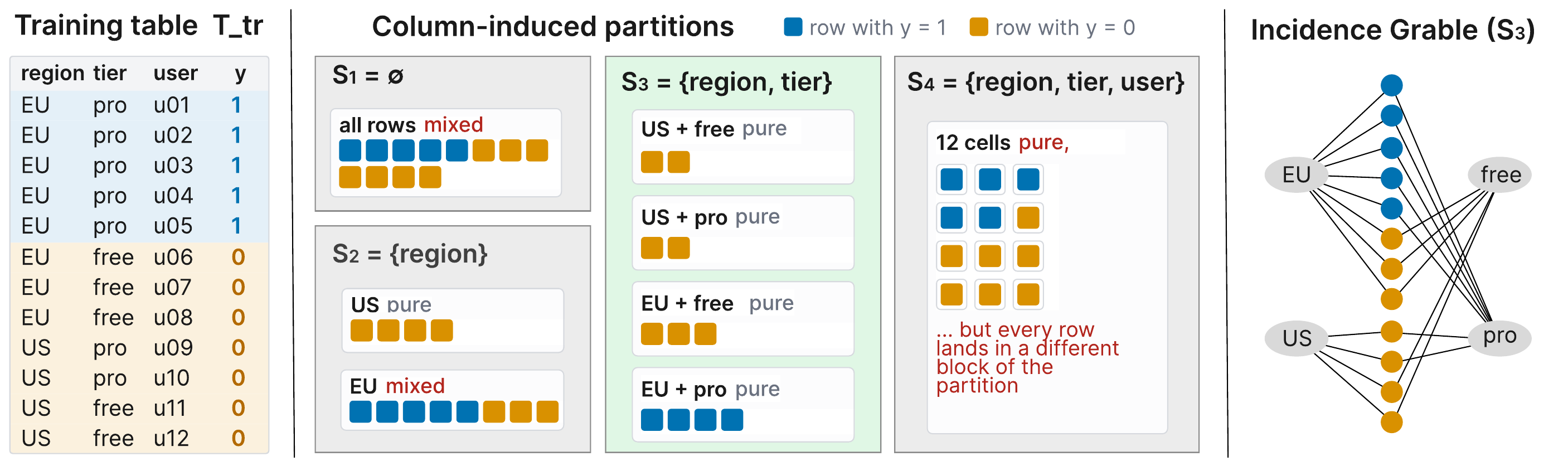}
    \caption{A candidate set $S$ partitions the training rows by their
    projection $r|_S$. Under-refinement leaves differently labelled rows in
    the same cell; over-refinement produces sparsely supported cells. After
    selecting $S^\star$, the chosen values become shared value nodes. %and the remaining attributes features of the row nodes.
    }
    \label{fig:alignment}
\end{figure}
which is monotone under refinement and ranges from
$n_{\mathrm{tr}}^{-1/2}$ (a single cell) to $1$ (all singletons).
Let
$\widehat{\mathrm{Risk}}_{\mathrm{val}}(\widehat h_S)$ be the average
validation loss of $\widehat h_S$ for some bounded loss function $\ell$. We score $S$ by
\begin{equation}
    \mathcal J(\pi_S)
    :=
    \widehat{\mathrm{Risk}}_{\mathrm{val}}(\widehat h_S)
    +
    \lambda\,\Omega(T_{\mathrm{tr}},\pi_S),
    \qquad \lambda\geq 0.
    \label{eq:main-alignment-score}
\end{equation}
The first term rewards distinctions that track the labels on held-out
rows; the second charges for distinctions supported by too few training
rows. Minimising $\mathcal J$ therefore selects the partition best aligned
with the labels: an attribute enters only when its held-out gain outweighs
the fragmentation it introduces.
The penalty is principled rather than heuristic. For binary classification,
let $\widehat g_S$ be the empirical majority classifier on the cells of
$\pi_S$, $\mathrm{Risk}^{\star}$ the unrestricted optimal population risk,
and $\mathrm{Risk}^{\star}_S$ the best population risk among classifiers
constant on those cells. For $S$ fixed independently of $T_{\mathrm{tr}}$,
with probability at least $1-\delta$,
\begin{equation*}
    \mathrm{Risk}(\widehat g_S)-\mathrm{Risk}^{\star}
    \leq
    \bigl(\mathrm{Risk}^{\star}_S-\mathrm{Risk}^{\star}\bigr)
    +\Omega(T_{\mathrm{tr}},\pi_S)
    +4\sqrt{\frac{\ln(4/\delta)}{2n_{\mathrm{tr}}}}.
    \label{eq:main-occupancy-bound}
\end{equation*}
Thus $\Omega$ bounds the estimation term in a generalisation bound:
refinement can lower approximation error, but raises the observable quantity
controlling estimation. Moreover, conditioning on $T_{\mathrm{tr}}$,
concentration on the validation sample is uniform over all $2^{|F|}$
candidates: with $\varepsilon_L
    :=
    L\sqrt{
        \frac{|F|\ln 2+\ln(2/\delta)}{2n_{\mathrm{val}}}
    }$,
an exact minimiser $S^\star$ of \eqref{eq:main-alignment-score} satisfies,
with probability at least $1-\delta$,
\[
    \mathrm{Risk}(\widehat h_{S^\star})
    +\lambda\Omega(T_{\mathrm{tr}},\pi_{S^\star})
    \leq
    \min_{S\subseteq F}
    \bigl\{
        \mathrm{Risk}(\widehat h_S)
        +\lambda\Omega(T_{\mathrm{tr}},\pi_S)
    \bigr\}
    +2\varepsilon_L.
\]
Because the concentration event is uniform over all $2^{|F|}$ candidates,
it also covers candidates inspected adaptively during search; if the search
returns an $\eta$-suboptimal empirical minimiser, the oracle inequality
holds with an additional additive term $\eta$. Full statements and proofs
are in \Cref{sec:score-details}.
\section{Finding an aligned attribute set}
\label{sec:autograble-algorithm}
The score $\mathcal J$ specifies which attribute sets are desirable; this
section addresses how to find one. Exact minimisation over all $2^{|F|}$
subsets is provably out of reach: even a requirement much weaker than alignment is already NP-complete.

\textbf{Alignment rather than maximal separation.}
An attribute set $S$ \emph{separates the labels} when
$r[Y]\neq r'[Y]$ implies $r|_S\neq r'|_S$ for all
$r,r'\in T_{\mathrm{tr}}$. Separation is necessary for a cell-constant predictor to fit the training labels, 
but it is not the objective,
 and it is attained by two degenerate choices: taking all attributes separates whenever any subset does, at the price of the finest and least supported partition, while a key-like attribute separates all rows, permitting memorisation while exposing no repeated structure to generalise from.
    
% but it is not the objective: it is attained by two degenerate constructions. Choosing all attributes separates whenever any subset does, and yields the largest (and least desirable) grable; choosing a key-like attribute separates all rows, permitting memorisation while exposing no repeated structure to generalise from.

The desired partition is \emph{aligned} with the task---fine enough to distinguish label-relevant row types, coarse enough to preserve support among rows that should share
information. Even so, under a budget $|S|\leq k$, the separation floor is computationally hard.
\begin{theorem}[Separation is NP-complete]
\label{thm:separate-hard}
Given $T_{\mathrm{tr}}$, $F$, and $k$, deciding whether some $S\subseteq F$
with $|S|\leq k$ separates the labels is NP-complete.
\end{theorem}
Unless $\mathrm{P}=\mathrm{NP}$, no polynomial-time exact method solves
even this restricted selection problem, let alone the alignment objective.
We therefore use greedy local search.
\algrenewcommand\algorithmicrequire{\textbf{Input:}}
\algrenewcommand\algorithmicensure{\textbf{Output:}}

\begin{algorithm}[t]
\caption{SCS: structural column selection}
\label{alg:autograble}
\begin{algorithmic}[1]
\Require $T = T_{\mathrm{tr}} \cup  T_{\mathrm{val}}$, label $Y$, candidates $F \subseteq A$
  ,
  direction $\in \{\mathrm{fwd},  \mathrm{bwd}\}$, signature $\sigma\in\{\mathrm{val},\mathrm{freq}\}$, $\lambda \ge 0$, tolerance $\tau \ge 0$
\Ensure selected columns $S^\star$
\If{$\sigma = \mathrm{freq}$}
    \State $T \gets \textsc{FreqEncode}(T, F)$ \quad// via \eqref{eq:freqencode}
\EndIf
\If{direction $= \mathrm{fwd}$}
  \State $S \gets \emptyset$;\quad
         $\mathrm{moves}(S) := \{\, S \cup \{c\} : c \in F \setminus S \,\}$
\Else
  \State $S \gets F$;\quad
         $\mathrm{moves}(S) := \{\, S \setminus \{c\} : c \in S \,\}$
\EndIf
\Repeat
  \State $S' \gets \arg\min_{S'' \in \mathrm{moves}(S)}\ \mathcal{J}(\pi_{S''})$
         \quad// via \eqref{eq:main-alignment-score}
  \If{$\mathcal{J}(\pi_{S}) - \mathcal{J}(\pi_{S'}) > \tau$}
    \State $S \gets S'$
  \Else\ \textbf{break}
  \EndIf
\Until{$\mathrm{moves}(S) = \emptyset$}
\State \Return $S^\star \gets S$
\end{algorithmic}
\end{algorithm}

\textbf{Greedy search.}
\Cref{alg:autograble} (Structural column selection, SCS) performs local search on the refinement lattice of column subsets: at each step it evaluates every one-column neighbour of the current set---one addition or removal, controlled by chosen direction---and moves to the neighbour with the lowest $\mathcal J$. Forward search starts from the coarsest partition ($S=\emptyset$) and is cheap, terminating as soon as no single addition clears the tolerance; it is the right choice when few columns are expected to matter, but it can reject a column that becomes informative only jointly with another. Backward search starts from the
finest available partition ($S=F$) and is more robust to such
interactions, since an informative pair is present from the start rather than added one column at a time, at the cost of evaluating more candidates
and relying on $\lambda$ to prune the extra columns it keeps. The tolerance $\tau$ sets the minimum improvement in $\mathcal J$ the search will act on, so it does not chase differences attributable to validation noise; $\tau=0$ recovers standard greedy descent. The signature $\sigma$ controls the initial table encoding (values or frequency), which is explained below. Because $\mathcal J$ requires no model fit---each evaluation is one group-by over $T_{\mathrm{tr}}$ and one pass over $T_{\mathrm{val}}$---a full run costs $O\bigl(|F|^2(n_{\mathrm{tr}}+n_{\mathrm{val}})\bigr)$ in the worst case. The uniform validation event covers the adaptively returned subset, and the empirical optimisation gap $\eta$ of greedy search enters the oracle inequality only additively. 

\paragraph{What the search scores.}
\SCS\ may be run on the table $T$ directly ($\sigma$=val) or on the
frequency-recoded copy  $\textsc{FreqEncode}(T,F)$ ($\sigma$=freq), which has the same rows and columns as $T$ and entries
\begin{equation}\label{eq:freqencode}
  \textsc{FreqEncode}(T,F)[r,c] := \bigl|\{\, r' \in T : r'[c] = r[c] \,\}\bigr|,
  \qquad c \in F .
\end{equation}
leaving the label and the attributes outside $F$ unchanged, so each cell in a candidate column is replaced by the number of times its value occurs in that column. This changes what a column is scored on: under value encoding, a column scores well when particular values track the label; under frequency encoding, when the multiplicity pattern it induces does---whether a value is rare or shared, not which value it is. The second is the weaker signal, and deliberately so: it selects columns for the structure they expose rather than for the labels their values happen to carry. This affects the score alone, since the constructor expands
the selected columns by their original values in either case. 

\textbf{Extension to relational databases.}
\label{sec:autograble-relational}
In the multi-table setting the subset search is unchanged; only the table changes. We materialise $\widetilde T:=T\bowtie_d\mathcal D$ by left joins along all foreign-key paths of length $< d$ and evaluate candidates on $\widetilde T$. Each expanded row retains the identifier of its originating primary row. To deal with join multiplicity, duplicate appearances of a primary row are collapsed within each cell, and if
row $i$ then occurs in $k_i(S)$ distinct cells, each appearance receives weight $1/k_i(S)$, so every primary row contributes one total unit. The selected attributes are then added to the graph: each becomes value nodes attached to the rows of the table it belongs to, and the foreign-key path from $T$ to that table is materialised. The skeleton is carried in by selection rather than fixed in advance, so selecting nothing returns the trivial grable ($\gamma_{\mathrm{triv}}$). 

\section{\autograble}
\label{sec:autograble}
\autograble is our table-to-graph constructor. Given a labelled table $T$
and candidate attributes $F$, it runs \SCS
(\Cref{sec:autograble-algorithm}) to obtain $S^\star$, then expands the
selected columns into a graph, the \emph{incidence grable}
$G_{S^\star}(T)$ (see Figure \ref{fig:alignment}), on which any row-level graph-learning method can be trained. This section supplies the graph-side vocabulary, grables and colour refinement, and proves the correspondence announced in \Cref{sec:preliminaries}: the partition $\pi_{S^\star}$ that \SCS optimises through group-bys is exactly the structural row partition of
$G_{S^\star}(T)$, the partition that bounds message-passing GNNs.

\textbf{Grables.}
A \emph{constructor} $\gamma$ maps a table to a typed, attributed graph
$\Ggamma{\gamma}$ with one distinguished \emph{row node} $v_r$ per row
$r$; the result is a \emph{grable}~\citep{grables}.\footnote{The trivial grable $\gamma_{\textsl{triv}}$ is the simplest grable, only consisting of row nodes and no edges.} A row-level GNN
predicts at the row nodes, so it can exploit information from other rows
only through the edges the constructor supplies. In this way a grable
turns a \emph{row-local} predictor, one whose output for a row is
unchanged when all other rows are removed, into an
\emph{extension-sensitive} one~\citep{grables}.

\textbf{Colour refinement.}
What a message-passing GNN can distinguish is governed by colour
refinement (or 1-WL). For a typed, attributed graph $G$, let $\chi_G^{(0)}(v)$
encode the initial node type and attributes exposed to the GNN, and
iterate $
    \chi_G^{(t+1)}(v)
    =
    \operatorname{HASH}\!\left(
        \chi_G^{(t)}(v),
        \{\!\!\{
            (\operatorname{type}(v,w),\chi_G^{(t)}(w))
            :w\in N_G(v)
        \}\!\!\}
    \right)$
with $\operatorname{HASH}$ injective. On a finite graph the induced
partition stabilises; $\pi_{\mathrm{CR}}(G)$ denotes its restriction to
the row nodes. Two row nodes with the same stable colour cannot receive
different outputs from any message-passing GNN, and sufficiently
expressive injective architectures attain this
bound~\citep{xu2018how,morris-WL-go-neural}. The partition
$\pi_{\mathrm{CR}}(G)$ is thus the expressive ceiling of graph learning
on $G$.

\textbf{The incidence grable.}
For $S\subseteq A$, the incidence grable $G_S(T):=\Ggamma{\ginc^S}$
contains
\begin{enumerate}
    \item one row node $v_r$ per $r\in T$, carrying the unexpanded
          attributes $r|_{A\setminus S}$; and
    \item one value node $u_{c,a}$ per occurring typed value $(c,a)$ with
          $c\in S$, whose initial feature identifies $(c,a)$.
\end{enumerate}
An edge of type $c$ joins $v_r$ and $u_{c,a}$ exactly when $r[c]=a$; the
label $Y$ is never exposed. Rows that agree on a selected value are thus
wired to a common value node, while the remaining attributes stay local
to their row. The typed-value encoding matters: were literal values
hidden, distinct values with identical structural profiles could remain
colour-equivalent.
To isolate the structure contributed by $S$ from the features that happen
to sit on the rows, let $G_S^{\circ}(T)$ be the
\emph{row-feature-erased reduct} of $G_S(T)$: all row nodes receive the
same initial colour, while value-node features and edge types are
retained. The reduct is an analysis device only.
\begin{lemma}
\label{lem:signature}
For all $r,r'\in T_{\mathrm{tr}}$,
% \[
%     \chi_{G_S^{\circ}(T_{\mathrm{tr}})}^{(\infty)}(v_r)
%     =
%     \chi_{G_S^{\circ}(T_{\mathrm{tr}})}^{(\infty)}(v_{r'})
%     \quad\Longleftrightarrow\quad
%     r|_S=r'|_S,
% \]
% and hence
$\pi_{\mathrm{CR}}\bigl(G_S^{\circ}(T_{\mathrm{tr}})\bigr)=\pi_S$.
\end{lemma}
% \begin{proof}
% If $r|_S\neq r'|_S$, then $r[c]\neq r'[c]$ for some $c\in S$, so through
% type-$c$ edges $v_r$ and $v_{r'}$ see the distinct value-node colours
% identifying $(c,r[c])$ and $(c,r'[c])$; colour refinement separates them
% after one round. Conversely, if $r|_S=r'|_S$, then $v_r$ and $v_{r'}$
% share the same initial colour and are adjacent to exactly the same value
% nodes through the same edge types; induction on the round gives
% $\chi^{(t)}(v_r)=\chi^{(t)}(v_{r'})$ for all $t$.
% \end{proof}
The lemma closes the loop. \SCS computes $\pi_S$ by a group-by, yet $\pi_S$ is precisely the stable row partition that the selected incidence structure induces once row features are suppressed. In minimising $\mathcal J$ over column subsets, \SCS is therefore
searching over the expressive ceilings of the corresponding
grables---without ever building a graph. Note that the cells are
structural row types, not connected components: rows in different cells
may share value nodes and exchange messages through them.

Selecting $S$ is a table-space operation, but the object it selects is not. By \Cref{lem:signature} the graph adds no distinguishing power over $\pi_S$; what it adds is access, within a cell, to evidence held by other rows---how many share a value, and what those rows carry. Such statistics can be materialised as columns one at a time, but which grouping to count over, and which row-local attribute it should interact with, are task-dependent. Expanding $S^\star$ into structure leaves that choice to the downstream learner rather than fixing it in advance.

\textbf{What the downstream GNN receives.}
The downstream learner receives the full graph $G_{S^\star}(T)$, not the reduct. Its row nodes retain $r|_{A\setminus S^\star}$, which may
separate rows within a structural cell, so
$\pi_{\mathrm{CR}}\bigl(G_{S^\star}(T)\bigr)\preceq\pi_{S^\star}$. The
selected columns fix the cross-row communication channels; the unexpanded attributes supply row-local evidence. The GNN combines both, and is in no way restricted to the block predictor \SCS used for scoring.

\section{Experiments}
\label{sec:experimental-section}
We ask three questions: \textbf{(RQ1)} Does \autograble\ recover the columns that generate a task and reject the rest? \textbf{(RQ2)} Is $\mathcal J$ an indicator of downstream performance for 1-WL-bounded learners? \textbf{(RQ3)} Under a fixed predictor, does \autograble\ build graphs that outperform other table-to-graph constructors? RQ1 requires known ground-truth structure and is answered on controlled tasks; RQ2 and RQ3 are answered on real tabular, transactional, and relational tasks.

% ============================================================
\subsection{Controlled tasks: does alignment select relevant structure?}
\label{sec:exp-controlled}

\textbf{Setup.}
RQ1 is about selection, not prediction, so no model is trained in this section: we run
\autograble{} to its selected subset $S^\star$ and compare that subset against the columns the task was built from. This requires ground truth, which is why the tasks are constructed.

\emph{Data and tasks.} The base table is Census/Adult \citep{adult-dataset}, whose categorical columns form the column universe $C$. A task is a choice of  $S^\star \subseteq C$ together with a label generated from $S^\star$ alone. We include row-local tasks: \emph{single-value}, \emph{conjunction}, and \emph{interaction}
(XOR) and extension-sensitive ones: \emph{count}, and \emph{duplicate}. Full generating processes are in Appendix~\ref{app:rq1-dgp-data-generation}.

\emph{Configurations.} We sweep the axes of \autograble: signature
(\emph{value}, \emph{frequency}), search direction (\emph{forward}, \emph{backward}), and $\lambda$. Table \ref{tab:rq1-main} reports the noise-free tasks at $\lambda \in \{0, 1\}$; the full sweep and the behaviour under label noise are in Appendix~\ref{app:extended-results}. Task instances are drawn once per (family, seed). Each cell is $N=10$ seeds, resampling the rows and the planted columns.

\emph{Metrics.} Let $S^\star$ be the selected columns and $S^{\circ}$ the
set the task was built from. Recovery $\mathrm{Rec} = |S^\star \cap S^{\circ}| /
|S^{\circ}|$ measures how much of the generating set is kept; $\mathrm{Exact} = \mathbbm{1}[S^\star =
S^{\circ}]$ additionally requires that nothing else is selected. We separate failure of the objective from failure of the search: an \emph{objective-metric mismatch} is a run in which $\mathcal J(\pi_{S^{\circ}}) \geq \mathcal J(\pi_{S^\star})$---the search stopped somewhere the score prefers to the generating set, so the score, not the search, is what ranked $S^{\circ}$ below---and a \emph{procedure failure} is a stopping point that is not optimal for $\mathcal J$. Only the second is a shortcoming of the algorithm rather than of the objective.
  
\begin{table*}[t]
\centering
\footnotesize
\caption{Column recovery on synthetic tasks, $N=10$ seeds. Rows show
$\lambda\in\{0,1\}$ (more values in App.~\ref{app:extended-results}). Cell format:
\%exact (\%recovery); $\emptyset$ marks a cell where the selected column
set was empty. A $\dagger$ marks an objective-metric mismatch
against ground truth ($\mathcal J(\text{true}) \geq \mathcal J(\text{reached})$).}
\label{tab:rq1-main}
\begin{tabular}{@{}llc ccc cc@{}}
\toprule
& & & \multicolumn{3}{c}{Row-local} & \multicolumn{2}{c}{Extension-sensitive} \\
\cmidrule(lr){4-6} \cmidrule(lr){7-8}
Encoding & Direction & $\lambda$ & Single-val. & Conj. & XOR & Count & Dup. \\
\midrule
\multirow{4}{*}{Value}
  & \multirow{2}{*}{Backward} & 0 & $0$ ($100$) & $30$ ($100$) & $90$ ($100$) & $0$ ($100$) & $0$ ($100$) \\
  &                           & 1 & $100$ ($100$) & $100$ ($100$) & $100$ ($100$) & $100$ ($100$) & $0$ ($100$) \\
\cmidrule(lr){2-8}
  & \multirow{2}{*}{Forward}  & 0 & $100$ ($100$) & $50$ ($100$) & $100$ ($100$) & $100$ ($100$) & $0$ ($0$) \\
  &                           & 1 & $100$ ($100$) & $50$ ($80$) & $100$ ($100$) & $100$ ($100$) & $0$ ($0$) \\
\midrule
\multirow{4}{*}{Freq.}
  & \multirow{2}{*}{Backward} & 0 & $0$ ($100$)$^\dagger$ & $0$ ($100$)$^\dagger$ & $0$ ($100$)$^\dagger$ & $0$ ($100$) & $0$ ($100$) \\
  &                           & 1 & $\emptyset$ (0)$^\dagger$ & $\emptyset$ (0)$^\dagger$ & $\emptyset$ (0)$^\dagger$ & $0$ (0) & $100$ ($100$)\\
\cmidrule(lr){2-8}
  & \multirow{2}{*}{Forward}  & 0 & $\emptyset$ (0)$^\dagger$ & $\emptyset$ (0)$^\dagger$ & $\emptyset$ (0)$^\dagger$ & 100 (100) & 100 (100) \\
  &                           & 1 & $\emptyset$ (0)$^\dagger$ & $\emptyset$ ($0$)$^\dagger$ & $\emptyset$ ($0$)$^\dagger$ & $100$ ($100$) & $100$ ($100$) \\
\bottomrule
\end{tabular}
\end{table*}

\textbf{Results. }
RQ1 asks whether \autograble recovers the columns that generate a task and rejects the rest. Table~\ref{tab:rq1-main} answers yes on both counts, but the two halves are controlled by different axes: recovery is decided by the signature, rejection by $\lambda$ and the search direction.

\emph{Recovery.} With the value signature \autograble recovers all tasks.
Recovering Count is consistent with it being extension-sensitive: selection only has to find the column, and on a fixed table the partition by that column's values already fixes each row's multiplicity: the graph is needed to compute the target, not to identify the column. Duplicate with forward direction fails for the opposite reason: its column is near-unique, so its value partition is almost all singletons. The frequency signature is more suited to recover it. Symmetrically, frequency encoding returns the empty set on the row-local families (except at backward with $\lambda = 0$), and the $\dagger$ marks confirm $\mathcal J(\emptyset)\le \mathcal J(S^\circ)$ there: frequencies discard value identity, which is what a row-local label depends on, so the score reports that no structure is worth building. Recovery is therefore complete when the signature matches the family, and the failures are informative rather than silent.

\emph{Exact recovery.} A gap between Rec and Exact means the selected set contains the true columns and more, and this is where $\lambda$ acts. It acts mainly on backward search, which keeps more columns than forward. At $\lambda=0$ backward search reaches Rec $=100$ with Exact between $0$ and $90$: the true columns are kept, but so are others. At $\lambda=1$ the gap closes, with Exact $=100$ on four of five families. This is $\Omega$ doing its job. The effect is not uniform: forward search also leaves a gap on Conjunction at $\lambda=0$ and there raising $\lambda$ stops the search earlier rather than dropping the extra column, so Rec falls to $80$ and Exact does not move; under frequency encoding with backward search, raising $\lambda$ collapses the row-local families from Rec $=100$ to the empty set. So $\lambda$ closes the Rec--Exact gap where the search starts too large, and costs Rec where it does not.
  
%\input{tables/table-encoding}

% ============================================================
\subsection{Real tasks: does alignment relate to downstream performance?}
\label{sec:exp-real}

\textbf{Setup. }
RQ2 and RQ3 need different apparatus. RQ2 compares $\mathcal J$ against downstream performance across many graphs for a single task; we take these from RDB2G-Bench
\citep{rdb2g}, which enumerates schema-level constructions for RelBench \cite{relbench-v1, relbench-v2} tasks and reports
trained-model performance for each. RQ3 compares construction methods, one graph per method, on the same task and split.

\emph{Fixed predictor.} We use GraphSAGE throughout, with fixed
architecture, budget and hyperparameters within regime (see Appendix \ref{appendix:details-hyperparameters-real-tasks}) for every construction and every dataset, and no per-constructor tuning. Under $\gamma_{\mathrm{triv}}$ there are no edges, so GraphSAGE reduces to an MLP on the row features. Differences in performance across constructors are attributable to the construction alone.

\emph{Constructors.} $\gamma_{\mathrm{triv}}$ is the no-structure floor. The fixed
constructions are the full incidence grable $\gamma_{\mathrm{inc}}$ \citep{grables},
which exposes every eligible column and so performs no selection, and, for foreign-key
schemas, the relational entity graph (REG) \citep{RDL-fey}, which takes the schema itself
as the graph (defined only in the relational case). The random constructor draws a column
subset uniformly at random in both size and membership, redrawn independently for each seed, and separates the contribution of \emph{which} columns are selected from
that of selecting any columns at all. The task-aware comparison is auGraph \citep{augraph}, which scores attributes individually and has multiple metrics. The one reported corresponds to the one with the best performance in the validation set (\texttt{gnn-gain} on all tasks studied). 

\emph{Datasets.} Three regimes. Transactional: FDB \citep{fraud-benchmark} single tables that are not i.i.d.\ at the table level---rows share device, merchant and account values. Relational: RelBench \citep{relbench-v1,relbench-v2}. Negative control: the i.i.d.\ single-table benchmark TabArena \citep{tabarena}, where each label should depend only on its own row. %, so $\gamma_{\mathrm{triv}}$ should suffice and no constructor should create a significant gain. 

\emph{Protocol.} Learning is transductive: one graph per dataset carrying train, validation and test row masks, so a node may aggregate from rows in other splits while labels are read only within its own. On the timestamped datasets splits are contiguous blocks in time. The neighbour sampler restricts each row to aggregate from rows at or before its own timestamp. For the RDB2G graph constructions, we compute $\mathcal J$ from the partition induced by colour refinement.

\textbf{Results.} $\mathcal J$ acts as a one-sided screen rather than a ranking: low $\mathcal J$ is necessary for strong downstream AUC on these candidates, not sufficient. On \texttt{driver-top3},
retaining the two lowest $\mathcal J$ levels discards $80\%$ of the candidate graphs and all ten of the best-performing graphs survive. The correlations
in Table~\ref{tab:kendall_tau_results} are consistent with this: they are negative throughout, but attenuated by ties, since $\mathcal J$ takes limited distinct values over the constructions (Figure~\ref{fig:scatterplot}) and the AUC values are themselves tightly clustered.
This is a property of the candidate set---RDB2G-Bench enumerates schema-level edits over a fixed key--foreign-key skeleton, so the induced row partitions differ in few ways. The association is weaker on test than on validation, as expected since $\mathcal J$ is also computed on the validation split.

\begin{figure}[t]
    \centering
    % Ensure NO blank lines exist between the first minipage and the second!
    \begin{minipage}[c]{0.55\textwidth}
        \centering
        \captionof{table}{Correlation (Spearman's $rho$ and Kendall's-tau $\tau_b$) between $\mathcal J(\cdot)$ and downstream performance (AUC on validation and test splits).  Graph constructions and performance obtained from RDB2G-Bench, using GraphSAGE as model and all seeds available (5 per graph for \texttt{study-outcome} and 15 for \texttt{driver-top3}).}
        \label{tab:kendall_tau_results}
        \begin{adjustbox}{max width=\linewidth}
        \begin{tabular}{ll cc cc}
    \toprule
    \multicolumn{2}{c}{} & \multicolumn{2}{c}{\texttt{relbench-f1}} & \multicolumn{2}{c}{\texttt{relbench-trial}} \\
    \multicolumn{2}{c}{} & \multicolumn{2}{c}{\texttt{driver-top3}} & \multicolumn{2}{c}{\texttt{study-outcome}} \\
    \cmidrule(lr){3-4} \cmidrule(lr){5-6}
%    \multicolumn{2}{l}{\textit{graphs / levels / max occ.}}
%      & \multicolumn{2}{c}{\textit{722 / L / xx\%}}
%      & \multicolumn{2}{c}{\textit{nnn / L / xx\%}} \\
    \cmidrule(lr){3-4} \cmidrule(lr){5-6}
    Split & $\lambda$ & $\rho$ & $\tau_b$ & $\rho$ & $\tau_b$ \\
    \midrule
    \multirow{2}{*}{Val}
      & 0 & $-0.229$  & $-0.183$ & $-0.458$  & $-0.291$ \\
      & 1 & $-0.235$  & $-0.192$ & $-0.508$  & $-0.372$ \\
    \midrule
    \multirow{2}{*}{Test}
      & 0 & $-0.118$  & $-0.082$ & $-0.287$ & $-0.219$ \\
      & 1 & $-0.158$  & $-0.091$ & $-0.251$ & $-0.203$ \\
    \bottomrule
\end{tabular}
        \end{adjustbox}
    \end{minipage}% <-- This percent sign is crucial to stop a line break
    \hfill% <-- Pushes the figure to the right edge
    \begin{minipage}[c]{0.40\textwidth}
        \centering
        \includegraphics[width=\linewidth]{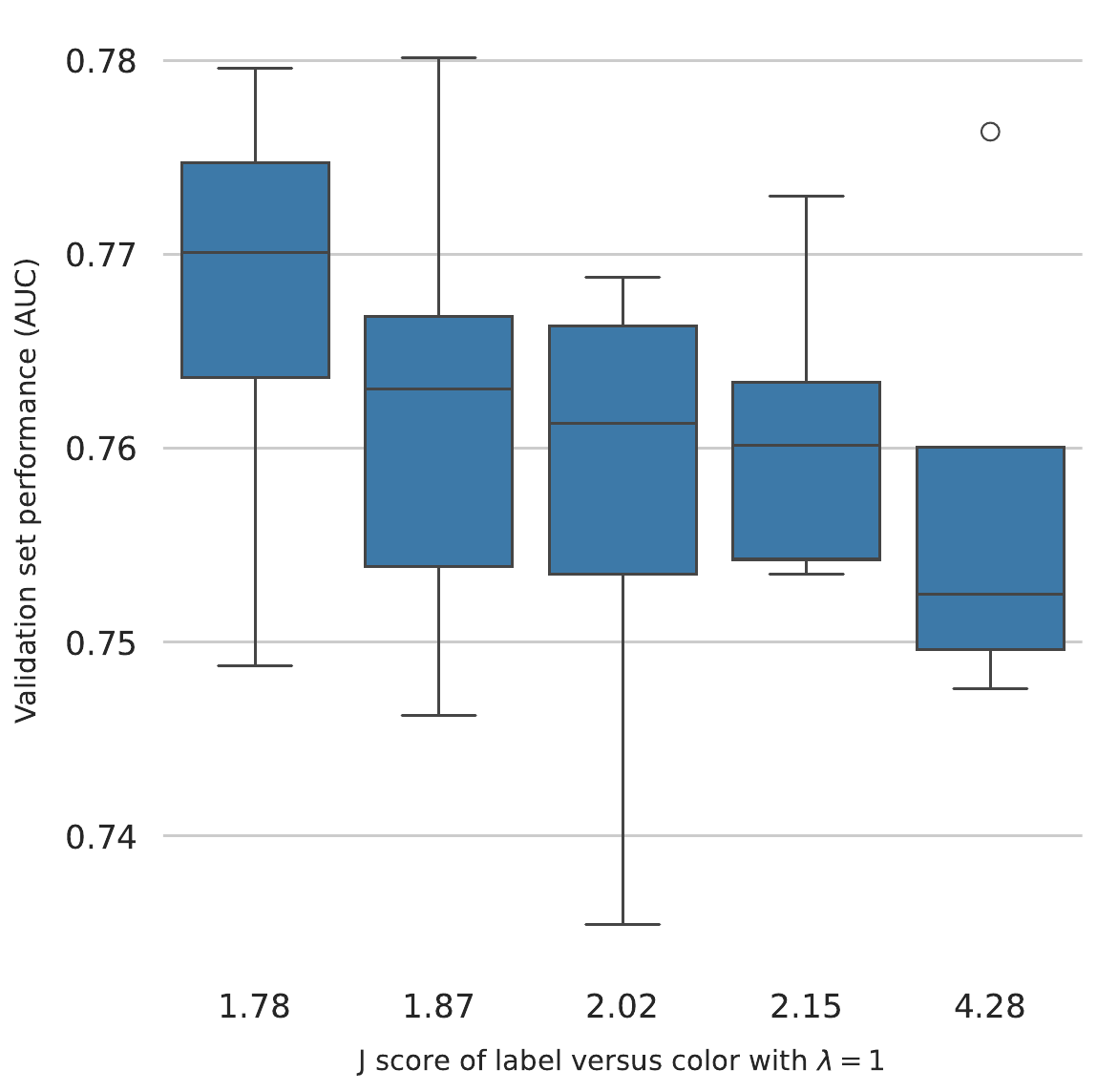}
        \captionof{figure}{$\mathcal J(\text{color}, \text{task})$ vs.  AUC on val set, task \texttt{driver-top3}, $74$ graphs, $15$ seeds per graph}.
        \label{fig:scatterplot}
    \end{minipage}
\end{figure}
% Requires: \usepackage{booktabs, graphicx}  % graphicx for \resizebox
% Optional: \usepackage{adjustbox} for finer width control
\newcommand{\pms}[2]{$#1{\scriptstyle\,\pm\,#2}$}                        % Original
\newcommand{\pmsB}[2]{$\mathbf{#1}{\scriptstyle\,\pm\,#2}$}              % Bold mean only
\newcommand{\pmsU}[2]{$\underline{#1}{\scriptstyle\,\pm\,#2}$}            % Underline mean only
\newcommand{\na}{\text{--}}  

\begin{table}[t]
\centering
\caption{Downstream test AUROC ($\uparrow$) under a \emph{fixed} predictor. Mean\,$\pm$\,std over $N{=}15$ seeds. Best in \textbf{bold},
second \underline{underlined}. \autograble reported results are for fixed setup (forward direction, $\lambda = 1$, $\sigma =$ frequencies). An $*$ marks equivalent constructions. More results in Tables \ref{tab:results-transactions-full} and \ref{tab:full-relbench}}
\label{tab:main-cls}
\resizebox{\textwidth}{!}{%
\begin{tabular}{l cc ccc}
\toprule
 & \multicolumn{2}{c}{Transactional (FDB)} & \multicolumn{3}{c}{Relational (RelBench)} \\
\cmidrule(lr){2-3}\cmidrule(lr){4-6}
Constructor & \texttt{vehicleloan} & \texttt{twitterbot}
& \texttt{driver-top3} & \texttt{driver-dnf} & \texttt{study-outcome} \\
\midrule
Trivial ($\gamma_{\mathrm{triv}}$)
  & \pmsU{0.647}{0.015} & \pms{0.864}{0.012}
  & \pms{0.674}{0.011} & \pms{0.596}{0.026} & \pmsB{0.677}{0.009}* \\
REG
  & \na & \na
  & \pms{0.777}{0.012} & \pms{0.733}{0.013} & \pms{0.635}{0.006} \\
Full incidence ($\gamma_{\mathrm{inc}}$)
  & \pms{0.619}{0.021} & \pms{0.831}{0.027}
  & \pms{0.746}{0.013} & \pms{0.687}{0.022} & \pms{0.621}{0.013} \\
Random
  & \pms{0.620}{0.019} & \pms{0.857}{0.021}
  & \pms{0.759}{0.025} & \pms{0.733}{0.012} & \pms{0.602}{0.022} \\
auGraph
  & \pms{0.631}{0.012} & \pms{0.887}{0.032}
  & \pmsU{0.791}{0.009} & \pmsU{0.742}{0.021} & \pms{0.630}{0.017} \\
\midrule
\autograble
  & \pmsB{0.662}{0.010} &  \pmsB{0.908}{0.007}
  & \pmsB{0.803}{0.011} & \pmsB{0.761}{0.021} & \pmsB{0.677}{0.009}* \\
  %\autograble+
  %& \pmsB{0.662}{0.010} &  \pmsB{0.919}{0.012}
  %& \pmsB{0.803}{0.011} & \pmsB{0.761}{0.021} & \pmsB{0.677}{0.009}* \\
\bottomrule
\end{tabular}%
}
\end{table}

\autograble uses $\mathcal J$ as a filter: it searches for a column subset $\mathcal J$ does not discard, and RQ3 asks whether the resulting graph beats the graphs the alternatives return, under the same predictor. Table~\ref{tab:main-cls} shows that it does, in both the single-table and the multi-table setting. More results for other datasets and tasks can be found in Appendix \ref{app:real-extended-results}.

The baselines fail in three ways. Fixed constructions do not select: exposing every eligible column ($\gamma_{\mathrm{inc}}$) loses to building no graph on three of five tasks, and the schema skeleton (REG) loses to $\gamma_{\mathrm{triv}}$ on \texttt{study-outcome}. Structure is not free, and a column that fragments the rows costs more than the signal it carries. Random selection separates how many columns from which ones: it beats full incidence on some tasks, so part of the gain comes from selecting fewer columns, but it stays behind task-aware selection. \autograble improves on auGraph everywhere, and at a lower cost: auGraph's best-performing metric, \emph{gnn-gain}, uses a GNN to score an attribute by materialising the augmented graph and running it over the validation nodes, which is done for every promotion.

The task \texttt{study-outcome} is special: the row-local predictor is the strongest baseline and every constructor that builds structure loses to it. \autograble returns $\emptyset$ and so $\gamma_{\mathrm{triv}}$ itself: no subset lowers validation risk enough to pay for the fragmentation it introduces. It is the only method compared here that can decline to build a graph. The same behaviour appears on the i.i.d.\ benchmark TabArena~\cite{tabarena}, where each label depends on its own row and structure should not help (App.~\ref{app:tabarena}).

Our code is available at: \url{https://github.com/TamaraCucumides/autoGrable}

%Finally, we consider options where the graph structure is \emph{learned}. This approaches require starting from a set of nodes and learn the adjacency matrix. As a representantive of this type of models we take a model that assigns learnable weights and can \emph{trim} some of the graph edges. We explore the effect of this mechanism on its own and combined with \autograble either as a previous or subsequent step. Figure \ref{} shows 

\section{Related work}
\label{sec:related-work}

\textbf{Expressive power as a design target.}
Message-passing networks separate two nodes only when colour refinement does~\citep{xu2018how,morris-WL-go-neural}; for relational data the canonical route to a graph is the incidence encoding, whose power under $1$-WL is characterised by~\citet{Gro+2020}. This literature is read almost always in one direction: given a graph, how much can a model distinguish, and how can an architecture distinguish more. Read the other way, the same bound is a statement about
capacity. \citet{morris2023wlvc} bound MPNNs' VC dimension by the number of colours $1$-WL produces, and what rules generalisation is whether the induced classes align with the labels rather than how many there are~\citep{li2025towards, maskey2026graph}. We apply this reading to the construction rather than the architecture. With the model class fixed at the $1$-WL bound, a construction reaches the learner only as a partition of the rows, so construction is partition design and finer is not necessarily better.

\textbf{Learning or modifying a graph.}
When no graph is given, one can learn one by  fixing the data points as nodes and inferring connectivity jointly with the predictor~\citep{franceschi,kazi2022dgm,fatemi2021slaps}; see \citet{zhou2023opengsl} for a survey. Two assumptions divide this
setting from ours. The node set is given---a table offers no such points, and what the
row set becomes as a graph, including its value nodes and their typing, is itself the
construction we study. And edge quality is read from the downstream objective, knowable
only by training the predictor; our criterion is a property of the induced row
partition, evaluated before any model is trained. A parallel line modifies a graph that
already exists, adding or deleting edges to improve information flow via
curvature~\citep{topping2022}, spectral gap~\citep{karhadkar2023}, or effective
resistance~\citep{black2023}. Those objectives are label-independent flow proxies over a
given topology; ours is a label-dependent property of a graph that does not yet exist.

\textbf{Graphs from tables and databases.}
Table-to-graph constructions include $k$-nearest-neighbour graphs over rows, row--value
bipartite graphs, and hypergraph variants~\citep{li2025gnn4tdl}; RDL instead fixes the entity graph induced by the key--foreign-key skeleton~\citep{RDL-fey} and develops architectures over it~\citep{chen2025relgnn,dwivedi2026relgt}. An empirical finding is that this skeleton is not the right graph and that not all of it helps: the extraction strategy materially changes performance~\citep{wang2024dbinfer}, the best schema-level graph model beats the standard heuristic by up to $10\%$ while being hard to identify without training~\citep{rdb2g}, and validation performance is an unreliable selector across relational architectures~\citep{chen2026relatron}. Our position is stronger than \emph{select carefully}: the graph adds no distinguishing power, so its value lies in enabling cross-row interactions only when \emph{it pays off}, which is why declining to build one is an admissible output.

\textbf{Automatic construction from tables.}
The closest methods propose a construction and judge it by a trained model. auGraph~\citep{augraph} promotes attributes to nodes by relevance scores, one of which trains a GNN per candidate; AutoG~\citep{autog} has an LLM propose schema transformations and selects among them with an oracle that trains a basket of GNNs. \autograble returns an object of the same class over the same input regime, and its search calls no trained graph model---a candidate costs one group-by and one pass over the validation split. 
%Choosing columns whose partition refines the labels is classically MIN-FEATURES~\citep{almuallim1994}, but the selected columns here become structure rather than features, and the class scored by $J$ is exactly what a $1$-WL-bounded learner can realise on the resulting graph. 

See Appendix~\ref{app:related} for an extended discussion on related work.

\section{Conclusions}
We asked what makes a graph good for a table, and answered it before any model is trained. For a learner bounded by 1-WL, a construction is visible only as a partition of the rows, so the design question is not \emph{how much a construction distinguishes} but \emph{how it groups}. On a fixed table the maximum distinguishing power is already available without any construction, and under the incidence construction the induced partition is the projection onto the selected columns (Lemma~\ref{lem:signature}). Construction therefore reduces to column selection, and the criterion that matters is alignment with the label rather than refinement.

We introduced a score that selects columns exposing label-relevant structure, and turned it into \autograble, which optimises it greedily since the weaker selection problem is NP-complete. Empirically, \autograble\ (i) recovers the relevant columns on controlled tasks, and (ii) produces graphs that match or improve on canonical and task-aware baselines under a fixed predictor, in both the single-table and the multi-table setting. It can also decline to build a graph when no construction pays for the fragmentation it introduces.

%\textbf{Limitations.}
%We study the incidence construction, where the induced partition is the projection onto the selected columns and construction therefore reduces to column selection. Nevertheless, for an arbitrary graph construction the partition is still defined, as the one induced by colour refinement, and alignment can still be stated against it.  We also fix the predictor class at the 1-WL bound, so a more expressive learner may exploit distinctions the partition does not record and that $\mathcal J$ consequently does not score. Finally, our analysis and experiments concern binary classification.

%\section*{Author Contributions}
%\section*{Acknowledgements}

% For natbib users:
\bibliographystyle{unsrtnat}
\bibliography{reference}
% For bibLaTeX users:
% \printbibliography

\newpage
\appendix

\section*{GenAI Usage Statement}
\label{app:genai}

In this work, we used generative AI tools for language editing and basic coding assistance. Specifically, we used Claude (Anthropic) for grammar correction,
copy-editing, and phrasing of the manuscript, and Claude Code (Anthropic) as a coding assistant during implementation and debugging of the \autograble codebase and the experimental pipeline. We have reviewed all AI-assisted work and take responsibility for the final content of this work, including text, claims, or artifacts produced with the aid of generative AI.

\section{Extended related work}\label{app:related}

This appendix expands Section~\ref{sec:related-work} along three axes that the main text compresses: the expressivity--capacity literature underlying $\Omega$, the range of table-to-graph constructions, and a per-invariant comparison against automatic constructors.

\subsection{Expressivity, capacity, and the partition}
Beyond the $1$-WL bound~\citep{xu2018how,morris-WL-go-neural}, \citet{barcelo2020} give the
logical characterisation of what MPNNs express, and \citet{Gro+2020} the analysis
of incidence encodings of relational structures that our $\gamma^{S}_{\mathrm{inc}}$
instantiates. This literature is read in one direction: given a graph, how much can a model distinguish, and how can an architecture distinguish more.

The capacity side of the same bound appears less in construction work.
\citet{morris2023wlvc} bound the VC dimension of MPNNs by the number of colours $1$-WL produces, and \citet{carrasco2026rademachercomplexitygraphneural} link the empirical Rademacher complexity of a class constant on the colour classes by the number of those classes and how the sample distributes across them. The colouring a construction induces is simultaneously what the learner can separate and what it can overfit. It follows that expressive power alone does not determine generalisation
\citep{franks2024margin,maskey2026graph}: what rules it is whether the induced classes align with the labels, a trade-off between intra-class concentration and inter-class separation~\citep{li2025towards}.

This is precisely the spirit of $\mathcal J$: the first term is what refinement buys, the second is what it costs, and the sum is minimised to balance the trade-off. Lemma~\ref{lem:occ-bound} supplies the corresponding statement in our setting, with $\Omega$ as the occupancy-resolved form of the complexity term.

Two differences are worth stating. The results from previous literature diagnose a fixed architecture on a given input, while we use it as the objective for choosing the input. They are also stated at the graph level, over a sample of graphs; our partition is on the row nodes of a single graph.

\subsection{Learning a graph, and modifying one}
Latent graph inference learns a soft adjacency over given nodes, either by a bi-level program over a distribution on edges~\citep{franceschi}, by similarity in a learned embedding~\citep{kazi2022dgm}, by alternating between structure and representation~\citep{chen2020idgl}, or with an auxiliary denoising
objective~\citep{fatemi2021slaps}; \citet{zhu2021gslsurvey} and \citet{zhou2023opengsl}
survey the area and its benchmarks. The presupposition throughout is geometric: an edge
means proximity in some feature space. A table's native relation is exact agreement on a categorical value, which is combinatorial, and its node set is not given. Rewiring
addresses a different failure mode---information flow over an existing topology~\citep{alon2021}, corrected by curvature-guided edge surgery~\citep{topping2022}, spectral-gap maximisation~\citep{karhadkar2023}, or effective-resistance criteria~\citep{black2023}. These criteria are computed from
topology alone and are label-independent, and could be applied downstream of any
constructor, including ours.

\subsection{Table-to-graph constructions}
\citet{li2025gnn4tdl} organise single-table constructions by which tabular elements
become nodes (instances, features, values, cells) and by how edges are created
(intrinsic, rule-based, learned, retrieval-based); our incidence grable is the
row--value intrinsic case. In the multi-table setting the relational entity
graph~\citep{RDL-fey} fixes rows as nodes and foreign keys as edges, and the RelBench
line~\citep{relbench-v1, relbench-v2} standardises tasks over it, with architectures exploiting its specific structure~\citep{chen2025relgnn,dwivedi2026relgt}.
Structure can also be induced implicitly through cross-row attention in tabular
foundation models~\citep{tabpfn_nature,qu2025tabicl}; there the induced structure
depends on tokenisation and architectural detail and is not available as an object one
can inspect, compare, or hold fixed, which is what the grable abstraction~\citep{grables} isolates. In the reverse direction, benchmarks with tabular features on given graphs~\citep{bazhenov2025GraphLand} probe whether structure is used at all; our negative control on TabArena \cite{tabarena} plays the analogous role from the tabular side.

\subsection{Which structure is worth having}
Three recent results converge on the same finding from different directions.

\citet{wang2024dbinfer} make the relational-to-graph extraction strategy an explicit benchmark axis (Row2Node, Row2N/E, dummy tables) and report that the choice materially changes performance. \citet{rdb2g} enumerate schema-level graph models over
$5$ databases and $12$ tasks, producing roughly $50$k graph--performance pairs, and find the best model improves on the common heuristic by up to $10\%$ while remaining hard to identify without repeated GNN training; \citet{chen2026relatron} find that relational deep learning does not consistently beat deep feature synthesis, that no single architecture dominates, and that validation performance is an unreliable selector, and respond with training-free task signals for routing. Their object of selection is the architecture given a graph; ours is the graph given an architecture (or a family of them).

\subsection{Automatic table to graph constructors}
Table~\ref{tab:invariants} compares the methods that produce a graph automatically. We separate them by what they return, what they require as input, and what a single candidate evaluation costs. 

\begin{table}[h]
\centering
\small
\caption{Automatic table-to-graph constructors by invariant. \emph{Candidate cost} is the work required to score one candidate construction during search.}
\label{tab:invariants}
\begin{tabular}{llll}
\toprule
Method & Output class & Input regime & Candidate cost  \\
\midrule
REG~\citep{RDL-fey}          & FK entity graph & FK schema           & ---                  \\
$\gamma_{\mathrm{inc}}$~\citep{grables} & incidence grable & table / schema & ---         \\
auGraph~\citep{augraph}      & augmented base graph & table / schema & GNN prediction  \\
AutoG~\citep{autog}  & graph from schema edits & schema + col. names & GNN basket (oracle) \\
%RelAgent~\citep{huang2026relagent} & SQL program + tab model & relational DB & predictor fit & n/a \\
\textsc{AutoGrable}          & incidence grable & table / schema & group-by + val.  \\
\bottomrule
\end{tabular}
\end{table}

AutoG's oracle uses early-stage validation performance averaged over a basket of GNNs,
which is cheaper than full training per candidate but is still a trained graph model in the loop. 

\subsection{Relation to feature selection and to explanation}
Selecting columns whose induced partition refines the label partition, viewed as a covering problem
over pairs with different labels, is Test Cover, which is Theorem~\ref{thm:separate-hard}. However,
two things separate our objective from standard filter feature selection.
\begin{enumerate}
    \item In our setting, the selected columns are not features but structure: they become value
    nodes, and what they contribute is cross-row communication rather than row-local signal.
    \item The partition scored by $\mathcal J$ is not a convenient proxy:
    it is exactly the structural row partition the selected construction induces, and hence the
    component of the downstream learner's expressive ceiling that the construction controls. The
    downstream GNN also receives the unexpanded row features, so it can separate rows within a cell;
    $\mathcal J$ scores what the construction contributes, not the full model.
\end{enumerate}
Selection cannot substitute construction---on extension-sensitive targets no row-local predictor over the selected columns can compute the label at any capacity~\citep{grables}.

\citet{rissaki2025views} are formally the closest work we are aware of: they select database views minimising deviation-from-determinacy plus $\lambda \cdot \mathrm{cost}$, and observe that even the
projection-only fragment reduces to hitting set. The objective has the same shape as $\mathcal J$---fidelity plus a regularised complexity term---and the same combinatorial core. The difference is directional: their views explain a model that has already been trained, ours chooses a construction before one exists.
\section{Connection to Test Cover and Min Features Problem}
\label{app:testcover}

Selecting columns so that the induced row partition separates the labels is the
\textsc{Min-Features} (or $k$-\textsc{Feature-Set}) problem, shown NP-complete by
\citet{davies1994np}. We record here an alternative derivation as an instance of the
classical \textsc{Test Cover} problem. We restate the result on the incidence constructions graph side. 

\subsection{Setup}
A table $T$ has a finite set of rows, a label column $Y$, and feature columns
$A$. Entries are real numbers; categorical columns are assumed one-hot encoded, so this is without loss of generality. For $S\subseteq A$ and a row $r$, write $r|_S$ for the restriction of $r$ to $S$, and let $\pi_S$ be the partition of
rows defined by $r\sim r' \iff r|_S=r'|_S$. Let $\pi_Y$ be the partition induced by the label column. A column $c$ \emph{separates} a pair $\{r,r'\}$ if $r[c]\neq r'[c]$, and a set $S$ separates the pair if some $c\in S$ does, i.e.\ $r|_S\neq r'|_S$.

\begin{definition}[Label-separating set]\label{def:sep}
$S\subseteq A$ is \emph{label-separating} if $\pi_S$ refines $\pi_Y$: equivalently, every pair of rows with distinct labels is separated by some
$c\in S$.
\end{definition}

%A $1$-WL--bounded predictor over the columns $S$ is constant on the blocks of $\pi_S$ (\cref{lem:signature}), so it can fit the labels only when $\pi_S$ refines $\pi_Y$. Label-separation is thus the exact condition under which the constructed graph admits a correct message-passing predictor: aligning the row partition with the labels. This motivates the selection problem.

\begin{definition}[\textsc{WL-Separate}]\label{def:wlalign}
\emph{Given} a table $T$ with label column $Y$, feature columns $A$, and
$k\in\mathbb{N}$; \emph{decide} whether there exists a label-separating $S\subseteq A$ with $|S|\le k$. The optimization version, \textsc{Min-Separate}, asks for a label-separating $S$ of minimum size.
\end{definition}

\subsection{The test cover problem}
\begin{definition}[Test cover]
\label{def:testcover}
Given a set of items $I$ and a collection of tests $\mathcal{T}\subseteq 2^{I}$, where a test $t$ \emph{separates} $\{i,i'\}$ iff $|t\cap\{i,i'\}|=1$; \emph{find} a smallest $\mathcal{T}'\subseteq\mathcal{T}$
such that every pair of items is separated by some $t\in\mathcal{T}'$.
\end{definition}

Test cover is set cover over the universe of item pairs $\binom{I}{2}$: a test covers exactly the pairs it separates. It is NP-complete
\cite{garey1979computers}. 

\subsection{Reduction and hardness}
\begin{lemma}[\textsc{Test Cover} $\le_p$ \textsc{WL-Separate}]\label{lem:reduction}
Given an instance $(I,\mathcal{T})$ with $I=\{1,\dots,m\}$ and
$\mathcal{T}=\{t_1,\dots,t_p\}$, construct a table $T$ with
\begin{itemize}
  \item one row $r_i$ per item $i\in I$;
  \item one feature column $c_j$ per test $t_j$, with
        $r_i[c_j]=1$ if $i\in t_j$ and $r_i[c_j]=0$ otherwise;
  \item a label column with a distinct label per row, $r_i[Y]=i$.
\end{itemize}
Then for every $S\subseteq\{c_1,\dots,c_p\}$, $S$ is label-separating iff $\{t_j : c_j\in S\}$ is a test cover of $(I,\mathcal{T})$. 
\end{lemma}

\begin{proof}
Since labels are distinct, every pair of rows must be separated, so
label-separation means every pair $\{r_i,r_{i'}\}$ is separated by some $c_j\in S$. By construction $c_j$ separates $\{r_i,r_{i'}\}$ iff $r_i[c_j]\neq
r_{i'}[c_j]$ iff exactly one of $i,i'$ lies in $t_j$ iff $t_j$ separates $\{i,i'\}$. Hence $S$ separates all pairs iff $\{t_j:c_j\in S\}$ covers all pairs.
\end{proof}

\begin{theorem}\label{thm:hardness}
\textsc{WL-Separate} is NP-complete. 
\end{theorem}

\begin{proof}
Membership: given $S$, checking that every pair of differently-labelled rows is separated takes polynomial time. Hardness follows from Lemma \ref{lem:reduction}, which is size-preserving and therefore transfers both the decision hardness and the optimal value of test cover. 
\end{proof}

\newcommand{\proofcheck}[1]{}
\section{Details of the training-free alignment score}
\label{sec:score-details}

\paragraph{Roadmap.}
The purpose of the alignment score is to select a subset of columns that exposes
label-relevant structure without fragmenting the data into cells that are too small
to support reliable estimation or useful message passing. The construction has four
steps. Given a column subset $S$, we
\begin{enumerate}
    \item partition the training rows according to their values on $S$;
    \item attach to each resulting cell the simplest possible predictor, namely its
          empirical label distribution;
    \item measure how strongly the partition fragments the training sample through an
          occupancy penalty; and
    \item combine this penalty with the predictor's loss on a held-out validation set.
\end{enumerate}
Thus, adding columns creates a more expressive predictor, but also produces finer
and potentially less well-supported cells. The score makes this trade-off explicit.
We first define the cells, predictor, and occupancy penalty. We then establish the
statistical meaning of the penalty and the validation guarantee for the resulting
score. Finally, we describe the greedy search used to optimise it in practice.

\subsection{From column subsets to cells}
Let
$\mathcal X=\mathrm{val}^{A}$ be the ambient row space over the non-label attributes
$A=C\setminus\{Y\}$, and let $\mathcal Y$ be the finite label alphabet.
We treat the data as an indexed sample
\[
    T=\bigl((X_i,Y_i)\bigr)_{i=1}^{n}.
\]
Fix a split of the indices into disjoint training and validation sets
$I_{\mathrm{tr}}$ and $I_{\mathrm{val}}$, with
\[
    n_{\mathrm{tr}}=|I_{\mathrm{tr}}|\geq 1,
    \qquad
    n_{\mathrm{val}}=|I_{\mathrm{val}}|\geq 1.
\]
The corresponding subsamples are denoted by
$T_{\mathrm{tr}}$ and $T_{\mathrm{val}}$.

Let $F\subseteq A$ be the fixed set of columns eligible for selection.
Throughout this section, $S\subseteq F$ denotes a candidate subset.

The information exposed by $S$ about a row $x\in\mathcal X$ is its
projection $x|_S$. Rows that agree on this projection belong to the same
\emph{cell}. More precisely, for every value tuple $u\in\mathrm{val}^{S}$,
define
\[
    B_{S,u}
    =
    \bigl\{
        i\in I_{\mathrm{tr}} : X_i|_S=u
    \bigr\},
    \qquad
    N_{S,u}=|B_{S,u}|.
\]
We index cells by value tuples rather than only by observed training rows.
Consequently, the count $N_{S,u}$ is defined even when no training row has
projection $u$. This will allow the predictor introduced below to be defined
on all of $\mathcal X$, including previously unseen rows.

The nonempty sets $B_{S,u}$ form a partition
$\pi_S\in\Pi(I_{\mathrm{tr}})$ of the training indices:
\[
    i\sim_{\pi_S}j
    \quad\Longleftrightarrow\quad
    X_i|_S=X_j|_S.
\]
We write $\pi'\preceq\pi$ when $\pi'$ refines $\pi$, meaning that every
block of $\pi'$ is contained in a block of $\pi$.

\subsection{The block predictor}

A column subset $S$ induces a predictor without fitting a parametric model.
On each occupied cell, the predictor simply returns the empirical label
distribution of the training rows in that cell. For $N_{S,u}>0$, define
\[
    \hat p_S(y\mid u)
    =
    \frac{1}{N_{S,u}}
    \sum_{i\in I_{\mathrm{tr}}}
    \mathbf 1\{X_i|_S=u,\ Y_i=y\}.
\]
For a cell that does not occur in the training data, these conditional
frequencies are not identifiable. We therefore use the training-set marginal
\[
    \hat p_0(y)
    =
    \frac{1}{n_{\mathrm{tr}}}
    \sum_{i\in I_{\mathrm{tr}}}
    \mathbf 1\{Y_i=y\}
\]
as a fixed fallback.

The resulting \emph{block predictor} is the function
\[
    \hat h_S(x)
    =
    \begin{cases}
        \hat p_S(\,\cdot\mid x|_S),
            & N_{S,x|_S}>0,\\[2pt]
        \hat p_0,
            & N_{S,x|_S}=0.
    \end{cases}
\]
The predictor is therefore defined on the entire row space $\mathcal X$, not
only on the rows observed during training. Moreover, conditional on
$T_{\mathrm{tr}}$, it is a fixed function: all its cell counts, label
frequencies, and fallback probabilities are determined by the training sample.

On occupied cells, $\hat p_S(\cdot\mid u)$ is the categorical
maximum-likelihood estimate. When a hard prediction is required, we use
\[
    \hat g_S(x)
    =
    \arg\max_{y\in\mathcal Y}\hat h_S(x)(y),
\]
with a fixed tie-breaking rule. Thus, $\hat h_S$ estimates a label
distribution, whereas $\hat g_S$ is the $0$--$1$ empirical risk minimiser
among classifiers that are constant on the cells induced by $S$.

\subsection{The occupancy penalty}

The block predictor becomes more expressive as cells are refined. In
particular, a sufficiently fine partition can place every training row in its
own cell, at which point the predictor merely memorises the training labels.

What matters, however, is not the number of cells that could exist, but how the observed
training rows are distributed over the cells that actually occur. We quantify
this through the occupancy functional
\[
    \Omega(T_{\mathrm{tr}},\pi)
    =
    \frac{1}{n_{\mathrm{tr}}}
    \sum_{B\in\pi}\sqrt{|B|}.
\]
For the partition induced by $S$, this becomes
\begin{equation}
    \Omega(T_{\mathrm{tr}},\pi_S)
    =
    \frac{1}{n_{\mathrm{tr}}}
    \sum_{u:N_{S,u}>0}\sqrt{N_{S,u}}
    =
    \frac{1}{n_{\mathrm{tr}}}
    \sum_{i\in I_{\mathrm{tr}}}
    \frac{1}{\sqrt{N_{S,X_i|_S}}}.
    \label{eq:occupancy}
\end{equation}
The final expression gives a useful row-wise interpretation: each training row
is charged the inverse square root of the number of rows supporting its
prediction.

Accordingly, $\Omega$ is best read as an \emph{occupancy penalty}. It is small
when many rows share a few well-populated cells, and large when the data is
spread over many sparsely populated cells. Since the block sizes always sum to
$n_{\mathrm{tr}}$, we have
\[
    \frac{1}{\sqrt{n_{\mathrm{tr}}}}
    \leq
    \Omega(T_{\mathrm{tr}},\pi_S)
    \leq
    1.
\]
The two extremes illustrate the scale:
\begin{itemize}
    \item If all rows belong to one cell, then
          $\Omega=1/\sqrt{n_{\mathrm{tr}}}$. Every prediction is supported by
          the entire training sample.
    \item If every cell is a singleton, then $\Omega=1$. This is the pure
          memorisation regime.
\end{itemize}
%
% More generally, if the training rows are distributed equally over $K$ cells,
% then
% \[
%     \Omega=\sqrt{\frac{K}{n_{\mathrm{tr}}}}.
% \]
% This is the familiar square-root estimation scale for a problem with $K$
% distinct row types observed with repetition.

\subsubsection{Adding columns increases fragmentation}

The first important property of $\Omega$ is structural: exposing more columns
can only refine the partition and therefore can only increase the penalty.

\begin{lemma}
\label{lem:col-mono}
If $S\subseteq S'\subseteq F$, then
\[
    \pi_{S'}\preceq\pi_S.
\]
\end{lemma}

\begin{proof}
If $X_i|_{S'}=X_j|_{S'}$, then $X_i|_S=X_j|_S$ because
$S\subseteq S'$. Hence every block of $\pi_{S'}$ is contained in a block of
$\pi_S$.
\proofcheck{Agreement on the larger set $S'$ includes agreement on every
column in the smaller set $S$. Thus adding columns may split an old block,
but it cannot merge rows from different old blocks.}
\end{proof}

\begin{lemma}
\label{lem:omega-mono}
If $\pi'\preceq\pi$, then
\[
    \Omega(T_{\mathrm{tr}},\pi)
    \leq
    \Omega(T_{\mathrm{tr}},\pi').
\]
\end{lemma}

\begin{proof}
Every block $B\in\pi$ is the disjoint union of the blocks
$B'\in\pi'$ that it contains. Therefore,
\[
    |B|
    =
    \sum_{\substack{B'\in\pi'\\B'\subseteq B}}|B'|.
\]
By subadditivity of the square root,
\[
    \sqrt{|B|}
    \leq
    \sum_{\substack{B'\in\pi'\\B'\subseteq B}}\sqrt{|B'|}.
\]
\proofcheck{For example, splitting a block of size $a+b$ replaces
$\sqrt{a+b}$ by $\sqrt a+\sqrt b$, which is never smaller. The displayed
inequality applies this observation to every piece of every block.}
Summing over $B\in\pi$ and dividing by $n_{\mathrm{tr}}$ proves the
claim.
\end{proof}

\begin{corollary}
\label{cor:omega-cols}
If $S\subseteq S'\subseteq F$, then
\[
    \Omega(T_{\mathrm{tr}},\pi_S)
    \leq
    \Omega(T_{\mathrm{tr}},\pi_{S'}).
\]
\end{corollary}

Thus, an additional column may improve predictive fidelity, but it can never
reduce the occupancy penalty. This is precisely the tension that the alignment
score will balance.

\subsubsection{Why occupancy is statistically meaningful}

The occupancy penalty is not merely a heuristic measure of fragmentation. It
also controls the estimation error of the corresponding class of cell-constant
classifiers.

Let
\[
    \mathcal G_S
    =
    \left\{
        g:\mathcal X\to\mathcal Y
        \;\middle|\;
        x|_S=x'|_S
        \Rightarrow
        g(x)=g(x')
    \right\}
\]
be the class of classifiers that are constant on the cells induced by $S$.
Let $\hat g_S$ be an empirical risk minimiser over $\mathcal G_S$ on
$T_{\mathrm{tr}}$ under $0$--$1$ loss. On every occupied cell, this is the
majority-label classifier. The convention above fixes its value on unoccupied
cells.

Write
\[
    \mathrm{Risk}(g)
    =
    \Pr_{\mathcal D}[g(X)\neq Y],
    \qquad
    \mathrm{Risk}^*
    =
    \inf_g\mathrm{Risk}(g),
\]
and
\[
    \mathrm{Risk}^*_S
    =
    \inf_{g\in\mathcal G_S}\mathrm{Risk}(g).
\]

\begin{lemma}
\label{lem:occ-bound}
Fix $S\subseteq F$ independently of $T_{\mathrm{tr}}$. Suppose that
$\mathcal Y=\{0,1\}$ and that $T_{\mathrm{tr}}$ consists of
$n_{\mathrm{tr}}$ independent draws from $\mathcal D$. For every
$\delta\in(0,1)$, with probability at least $1-\delta$ over
$T_{\mathrm{tr}}$,
\begin{equation}
    \mathrm{Risk}(\hat g_S)-\mathrm{Risk}^*
    \leq
    \underbrace{
        \bigl(\mathrm{Risk}^*_S-\mathrm{Risk}^*\bigr)
    }_{\text{approximation error}}
    +
    \underbrace{
        \Omega(T_{\mathrm{tr}},\pi_S)
    }_{\text{estimation error}}
    +
    4\sqrt{
        \frac{\ln(4/\delta)}{2n_{\mathrm{tr}}}
    }.
    \tag{$\ast$}
    \label{eq:occ-recall}
\end{equation}
\end{lemma}

\begin{proof}
Write $n=n_{\mathrm{tr}}$. Encode classifiers and labels by signs:
\[
    \tilde g=2g-1\in\{-1,+1\},
    \qquad
    \tilde Y_i=2Y_i-1.
\]
Then
\begin{equation}
    \mathbf 1\{g(X_i)\neq Y_i\}
    =
    \frac{1}{2}
    \bigl(1-\tilde Y_i\tilde g(X_i)\bigr).
    \label{eq:pm1}
\end{equation}
Let $\sigma_1,\ldots,\sigma_n$ be independent Rademacher signs and let
\[
    \widetilde{\mathcal G}_S
    =
    \{\tilde g:g\in\mathcal G_S\}.
\]

Every function in $\widetilde{\mathcal G}_S$ is constant on each block of
$\pi_S$. Its values on the training sample are therefore determined by one
freely chosen sign per block. Hence
\[
\begin{aligned}
    \widehat{\mathfrak R}
    (\widetilde{\mathcal G}_S)
    &=
    \frac{1}{n}
    \mathbb E_{\sigma}
    \left[
        \sup_{\tilde g\in\widetilde{\mathcal G}_S}
        \sum_{i=1}^{n}\sigma_i\tilde g(X_i)
    \right]                                                   \\
    &=
    \frac{1}{n}
    \sum_{B\in\pi_S}
    \mathbb E_{\sigma}
    \left|
        \sum_{i\in B}\sigma_i
    \right|.
\end{aligned}
\]
\proofcheck{Inside a fixed block $B$, the classifier must use one sign
$s_B\in\{-1,+1\}$. The best choice is the sign of
$\sum_{i\in B}\sigma_i$, so maximising over $s_B$ produces the absolute
value $\lvert\sum_{i\in B}\sigma_i\rvert$. Different blocks can choose their
signs independently, which explains the sum over blocks.}
By Cauchy--Schwarz applied to the random variables
$\lvert\sum_{i\in B}\sigma_i\rvert$ and $1$, together with
\[
    \mathbb E_{\sigma}
    \left(
        \sum_{i\in B}\sigma_i
    \right)^2
    =
    |B|,
\]
we obtain
\[
    \mathbb E_{\sigma}
    \left|
        \sum_{i\in B}\sigma_i
    \right|
    \leq
    \sqrt{|B|}.
\]
Consequently,
\[
    \widehat{\mathfrak R}
    (\widetilde{\mathcal G}_S)
    \leq
    \frac{1}{n}
    \sum_{B\in\pi_S}\sqrt{|B|}
    =
    \Omega(T_{\mathrm{tr}},\pi_S).
\]
\proofcheck{A block containing $|B|$ rows contributes at most
$\sqrt{|B|}/n$. Summing these contributions gives exactly the occupancy
penalty, so $\Omega$ is an observable upper bound on the complexity of the
cell-constant classifier class.}

Let
\[
    \mathcal L_S
    =
    \left\{
        (x,y)\mapsto\mathbf 1\{g(x)\neq y\}
        :g\in\mathcal G_S
    \right\}.
\]
By \eqref{eq:pm1}, the additive constant contributes
$\frac{1}{2n}\mathbb E_\sigma\sum_i\sigma_i=0$, while the multiplicative
factor $1/2$ remains. Moreover, the products $\sigma_i\tilde Y_i$ are again
independent Rademacher signs conditional on the sample. Therefore,
\[
    \widehat{\mathfrak R}(\mathcal L_S)
    =
    \frac{1}{2}
    \widehat{\mathfrak R}(\widetilde{\mathcal G}_S)
    \leq
    \frac{1}{2}\Omega(T_{\mathrm{tr}},\pi_S).
\]
\proofcheck{The loss is an affine transformation of the signed classifier.
The additive part averages to zero against the random signs, while the
multiplicative factor $1/2$ remains.}

Apply the standard one-sided empirical Rademacher bound for $[0,1]$-valued
loss classes with failure probability $\delta/2$. With probability at least
$1-\delta/2$, simultaneously for every $g\in\mathcal G_S$,
\[
\begin{aligned}
    \mathrm{Risk}(g)
    &\leq
    \widehat{\mathrm{Risk}}_{\mathrm{tr}}(g)
    +
    2\widehat{\mathfrak R}(\mathcal L_S)
    +
    3\sqrt{\frac{\ln(4/\delta)}{2n}}                         \\
    &\leq
    \widehat{\mathrm{Risk}}_{\mathrm{tr}}(g)
    +
    \Omega(T_{\mathrm{tr}},\pi_S)
    +
    3\sqrt{\frac{\ln(4/\delta)}{2n}}.
\end{aligned}
\]
For $\varepsilon>0$, choose a fixed $\varepsilon$-minimiser
$g^*_{S,\varepsilon}\in\mathcal G_S$ satisfying
\[
    \mathrm{Risk}(g^*_{S,\varepsilon})
    \leq
    \mathrm{Risk}^*_S+\varepsilon.
\]
Since this comparator is fixed independently of $T_{\mathrm{tr}}$, the
one-sided Hoeffding inequality, also with failure probability $\delta/2$,
gives
\[
    \widehat{\mathrm{Risk}}_{\mathrm{tr}}(g^*_{S,\varepsilon})
    \leq
    \mathrm{Risk}(g^*_{S,\varepsilon})
    +
    \sqrt{\frac{\ln(2/\delta)}{2n}}.
\]
By a union bound, both displayed inequalities hold with probability at least
$1-\delta$. Put
\[
    c
    =
    3\sqrt{\frac{\ln(4/\delta)}{2n}},
    \qquad
    b
    =
    \sqrt{\frac{\ln(2/\delta)}{2n}},
    \qquad
    \Omega
    =
    \Omega(T_{\mathrm{tr}},\pi_S).
\]
On the event above,
\[
\begin{aligned}
    \mathrm{Risk}(\hat g_S)
    &\leq
    \widehat{\mathrm{Risk}}_{\mathrm{tr}}(\hat g_S)
    +\Omega+c                                                    \\
    &\leq
    \widehat{\mathrm{Risk}}_{\mathrm{tr}}(g^*_{S,\varepsilon})
    +\Omega+c                                                    \\
    &\leq
    \mathrm{Risk}^*_S+\varepsilon+\Omega+c+b                    \\
    &\leq
    \mathrm{Risk}^*_S+\varepsilon+\Omega
    +4\sqrt{\frac{\ln(4/\delta)}{2n}}.
\end{aligned}
\]
\proofcheck{The three inequalities respectively use: the uniform
Rademacher bound for $\hat g_S$, empirical optimality of $\hat g_S$, and
Hoeffding's bound for the single fixed comparator
$g^*_{S,\varepsilon}$. Only the first step pays the occupancy term, which is
why the final bound contains one copy of $\Omega$.}
The second inequality uses the empirical optimality of $\hat g_S$.
Subtracting $\mathrm{Risk}^*$ and letting $\varepsilon\downarrow0$ gives
\eqref{eq:occ-recall}.
\end{proof}

The bound separates two effects. The approximation term measures what is lost
by forcing rows in the same cell to receive the same label. The occupancy term
measures the estimation cost of learning one prediction per occupied cell.
Thus, finer partitions can reduce approximation error, but they necessarily
increase the quantity controlling estimation error.
Because $\Omega$ can be as large as $1$, the bound may be vacuous in the
singleton-cell memorisation regime. This is precisely the regime that the
occupancy penalty is designed to detect.

The randomness in \Cref{lem:occ-bound} is over $T_{\mathrm{tr}}$: the lemma
explains why the occupancy of the training sample is worth penalising. The
validation result below instead conditions on $T_{\mathrm{tr}}$ and
randomises over $T_{\mathrm{val}}$: it certifies the selection made using
that penalty.

\subsection{The alignment score}

The occupancy penalty alone would always favour the coarsest possible
partition. We therefore balance it against predictive fidelity on the
validation split.

Let $\ell$ be a loss taking values in $[0,L]$. Define
\[
    \widehat{\mathrm{Risk}}_{\mathrm{val}}(\hat h_S)
    =
    \frac{1}{n_{\mathrm{val}}}
    \sum_{i\in I_{\mathrm{val}}}
    \ell\bigl(\hat h_S(X_i),Y_i\bigr).
\]
The training-free alignment score is
\begin{equation}
    \mathcal J(\pi_S)
    =
    \widehat{\mathrm{Risk}}_{\mathrm{val}}(\hat h_S)
    +
    \lambda\,
    \Omega(T_{\mathrm{tr}},\pi_S),
    \qquad
    \lambda\geq 0.
    \label{eq:objective2}
\end{equation}

The two terms have deliberately different roles:
\begin{itemize}
    \item $T_{\mathrm{tr}}$ determines the cells, the block predictor, and
          the occupancy penalty;
    \item $T_{\mathrm{val}}$ measures how well the resulting predictor
          generalises beyond those training cells.
\end{itemize}
The parameter $\lambda$ states how much validation loss we are willing to
trade for a coarser and better-supported structure.

\subsection{Validation guarantee}

Because every predictor $\hat h_S$ is fixed after conditioning on
$T_{\mathrm{tr}}$, standard concentration on the independent validation set
gives a uniform guarantee over all candidate subsets.

\begin{proposition}[Uniform validation guarantee]
\label{prop:val-oracle}
Fix $\delta\in(0,1)$ and condition on $T_{\mathrm{tr}}$. Assume that
\begin{enumerate}
    \item[\textup{(A1)}]
    the validation rows are independent draws from a distribution
    $\mathcal D_{\mathrm{val}}$, independent of $T_{\mathrm{tr}}$, and
    \[
        \mathrm{Risk}(h)
        =
        \mathbb E_{\mathcal D_{\mathrm{val}}}
        [\ell(h(X),Y)];
    \]
    \item[\textup{(A2)}]
    the loss $\ell$ takes values in $[0,L]$; and
    \item[\textup{(A3)}]
    the candidate set $F\subseteq A$ is fixed independently of
    $T_{\mathrm{val}}$.
\end{enumerate}
Define
\[
    \varepsilon_L
    =
    L
    \sqrt{
        \frac{|F|\ln 2+\ln(2/\delta)}
             {2n_{\mathrm{val}}}
    }.
\]
Then, with probability at least $1-\delta$ over $T_{\mathrm{val}}$, every
minimiser
\[
    S^\star
    \in
    \arg\min_{S\subseteq F}\mathcal J(\pi_S)
\]
satisfies
\begin{equation}
\begin{aligned}
    \mathrm{Risk}(\hat h_{S^\star})
    +
    \lambda\Omega(T_{\mathrm{tr}},\pi_{S^\star})
    \leq
    \min_{S\subseteq F}
    \Bigl\{
        \mathrm{Risk}(\hat h_S)
        +
        \lambda\Omega(T_{\mathrm{tr}},\pi_S)
    \Bigr\}
    +
    2\varepsilon_L.
\end{aligned}
\label{eq:validation-oracle}
\end{equation}
In particular,
\[
    \mathrm{Risk}(\hat h_{S^\star})
    \leq
    \min_{S\subseteq F}
    \Bigl\{
        \mathrm{Risk}(\hat h_S)
        +
        \lambda\Omega(T_{\mathrm{tr}},\pi_S)
    \Bigr\}
    +
    2\varepsilon_L.
\]
\end{proposition}

\begin{proof}
Condition on $T_{\mathrm{tr}}$. For every $S\subseteq F$, the function
$\hat h_S$ is then fixed on all of $\mathcal X$: its counts, empirical
label distributions, and fallback distribution are determined entirely by the
training sample.
\proofcheck{Conditioning freezes everything learned from the training data.
The validation rows are then the only remaining source of randomness, so
ordinary concentration inequalities apply to each fixed predictor.}

Because $T_{\mathrm{val}}$ is independent of $T_{\mathrm{tr}}$, the random
variables
\[
    \ell(\hat h_S(X_i),Y_i),
    \qquad
    i\in I_{\mathrm{val}},
\]
are independent, take values in $[0,L]$, and have mean
$\mathrm{Risk}(\hat h_S)$. Hoeffding's inequality therefore gives
\[
    \Pr\left[
        \left|
            \widehat{\mathrm{Risk}}_{\mathrm{val}}(\hat h_S)
            -
            \mathrm{Risk}(\hat h_S)
        \right|
        >t
    \right]
    \leq
    2\exp\left(
        -\frac{2n_{\mathrm{val}}t^2}{L^2}
    \right).
\]

By assumption \textup{(A3)}, the family
\[
    \{\hat h_S:S\subseteq F\}
\]
is fixed before the validation set is observed and contains at most
$2^{|F|}$ predictors. A union bound with total failure probability
$\delta$ therefore yields $t=\varepsilon_L$. Hence, with probability at
least $1-\delta$, simultaneously for all $S\subseteq F$,
\begin{equation}
    \left|
        \widehat{\mathrm{Risk}}_{\mathrm{val}}(\hat h_S)
        -
        \mathrm{Risk}(\hat h_S)
    \right|
    \leq
    \varepsilon_L.
    \label{eq:two-sided}
\end{equation}
\proofcheck{There are at most $2^{|F|}$ subsets. Hoeffding controls one
subset, and the union bound makes the control simultaneous over all subsets.
The term $|F|\ln 2$ is simply $\ln(2^{|F|})$.}

On this event, for every $S\subseteq F$, the optimality of $S^\star$ for
$\mathcal J$ and two applications of \eqref{eq:two-sided} give
\[
\begin{aligned}
    \mathrm{Risk}(\hat h_{S^\star})
    +
    \lambda\Omega(T_{\mathrm{tr}},\pi_{S^\star})
    &\leq
    \widehat{\mathrm{Risk}}_{\mathrm{val}}(\hat h_{S^\star})
    +
    \lambda\Omega(T_{\mathrm{tr}},\pi_{S^\star})
    +
    \varepsilon_L                                             \\
    &\leq
    \widehat{\mathrm{Risk}}_{\mathrm{val}}(\hat h_S)
    +
    \lambda\Omega(T_{\mathrm{tr}},\pi_S)
    +
    \varepsilon_L                                             \\
    &\leq
    \mathrm{Risk}(\hat h_S)
    +
    \lambda\Omega(T_{\mathrm{tr}},\pi_S)
    +
    2\varepsilon_L.
\end{aligned}
\]
\proofcheck{The first and third lines replace population risk by validation
risk and back again, costing $\varepsilon_L$ each time. The middle line uses
only that $S^\star$ minimises the empirical penalised score. This is the
source of the final $2\varepsilon_L$.}
Minimising the right-hand side over $S\subseteq F$ proves
\eqref{eq:validation-oracle}. Dropping the nonnegative penalty on the left
gives the final statement.
\end{proof}

Assumption \textup{(A1)} holds for a uniformly random train--validation split.
For a time-based or grouped split, the guarantee should instead be read
relative to the validation distribution that is actually sampled. If rows
within a validation group are dependent, Hoeffding's inequality must be
replaced by an appropriate block-wise concentration inequality.

\subsection{How to read the guarantee}

\paragraph{The comparator is penalised risk.}
For $\lambda>0$, \Cref{prop:val-oracle} does not claim that $S^\star$ is
near-optimal for prediction risk alone. It says that $S^\star$ is
near-optimal for the trade-off declared by $\mathcal J$. A subset with
slightly larger prediction risk may therefore be preferred when it induces
substantially better-supported cells. When $\lambda=0$, the penalty
disappears and the result reduces to the classical held-out model-selection
bound.

\paragraph{Why use a positive penalty?}
There are two reasons to choose $\lambda>0$.

First, the validation guarantee conditions on one realised training sample.
A partition consisting largely of singleton cells can behave well for that
particular draw while changing sharply when the training sample changes.
\Cref{lem:occ-bound} shows that occupancy is precisely the quantity controlling
this sensitivity.

Second, $S^\star$ is not selected only to make predictions. It is used to
construct $\gamma_{\mathrm{inc}}^{S^\star}$, in which the cells determine
shared neighbourhoods. When $\Omega$ is close to $1$, most row nodes have
essentially unique neighbourhoods, leaving message passing with little shared
structure to aggregate. The score is therefore a \emph{structural selector},
not merely a risk estimator.

\paragraph{Why is occupancy measured on the training set?}
The block predictor is fitted from $T_{\mathrm{tr}}$, so its estimation error
is governed by the occupancy of the cells in $T_{\mathrm{tr}}$. Measuring
$\Omega$ on $T_{\mathrm{val}}$ would charge the score for rows that were
never used to estimate $\hat p_S(\cdot\mid u)$, while failing to penalise
training-time memorisation directly.

Using $T_{\mathrm{tr}}$ also makes both the predictor and the penalty fixed
after conditioning on the training sample. This conditional viewpoint is what
the proof of \Cref{prop:val-oracle} requires. It does not require the predictor
and penalty to be unconditionally independent; indeed, both are functions of
$T_{\mathrm{tr}}$.

\paragraph{What does the guarantee say about greedy search?}
The union bound in \Cref{prop:val-oracle} ranges over the complete lattice
$\{S:S\subseteq F\}$. Consequently, the uniform deviation event also covers
a subset chosen adaptively by inspecting validation scores.

In particular, if $\tilde S$ is the subset returned by
\Cref{alg:autograble}, then
\[
    \left|
        \widehat{\mathrm{Risk}}_{\mathrm{val}}(\hat h_{\tilde S})
        -
        \mathrm{Risk}(\hat h_{\tilde S})
    \right|
    \leq
    \varepsilon_L.
\]
The deviation for one returned candidate is $\varepsilon_L$; the factor
$2\varepsilon_L$ in \Cref{prop:val-oracle} arises only when two candidates
are compared.

More generally, suppose that $\tilde S$ is $\eta$-suboptimal for the
empirical score:
\[
    \mathcal J(\pi_{\tilde S})
    \leq
    \min_{S\subseteq F}\mathcal J(\pi_S)+\eta.
\]
Then the conclusion of \Cref{prop:val-oracle} holds with
$\eta+2\varepsilon_L$ in place of $2\varepsilon_L$. The validation
guarantee therefore survives approximate optimisation. What it does not
provide is a bound on $\eta$: although the occupancy term is monotone, the
validation-risk term is not, and greedy search may stop at a local minimum.

\subsection{Greedy optimisation}
\label{sec:algorithm}

We provide a bit more background on \Cref{alg:autograble}.
As mentioned, exact minimisation of \eqref{eq:main-alignment-score} is generally infeasible. The search
space contains $2^{|F|}$ column subsets, and the underlying separation
problem is NP-complete. We therefore use the greedy local search in

\paragraph{Search strategy.}
The algorithm maintains a current subset $S$ and considers all subsets that
differ from $S$ by one column.

Forward selection starts from $S=\emptyset$ and considers additions
$S\cup\{c\}$. Backward elimination starts from $S=F$ and considers
deletions $S\setminus\{c\}$. At every iteration, the algorithm evaluates
$\mathcal J$ on all one-column neighbours and selects the best one. The move
is accepted only when it improves the current score by more than a tolerance
$\tau$. Otherwise, the search terminates.

The algorithm returns the surviving subset $S^\star$, together with the
block model already computed for it. This block model is then passed to the
constructor as $\gamma_{\mathrm{inc}}^{S^\star}$.

\paragraph{Choosing the search direction.}
The two directions traverse the refinement lattice in opposite ways.
Forward selection starts from the coarsest partition: it has the smallest
possible occupancy penalty but is also the least expressive. Backward
elimination starts from the finest available partition: it is the most
expressive but pays the largest occupancy penalty. The validation-risk term is
not monotone, so the two directions may reach different local minima.

Forward selection is typically cheaper. It often terminates after only a few
accepted additions and never needs to construct the finest partition. It is
therefore preferable when only a small number of columns are expected to
matter.

Backward elimination is more robust to interactions between attributes. A
column that is uninformative in isolation but useful together with another
column is present at the start and is not rejected before that interaction can
be evaluated. Backward search also tends to remove high-cardinality columns
early, since such columns strongly increase $\Omega$ and often fragment the
cells on which the block predictor relies.

\paragraph{Choosing the tolerance.}
The tolerance $\tau$ prevents the search from accepting improvements that
are indistinguishable from validation noise. \Cref{prop:val-oracle} supplies
the natural scale:
\[
    \varepsilon_L
    =
    O\left(
        L\sqrt{
            \frac{|F|+\log(1/\delta)}
                 {n_{\mathrm{val}}}
        }
    \right).
\]
A single candidate's validation risk is accurate to
$\varepsilon_L$, while the uncertainty in a pairwise comparison can be as
large as $2\varepsilon_L$. Choosing $\tau$ on this order prevents the
search from pursuing sampling fluctuations. Setting $\tau=0$ recovers ordinary greedy descent.

\paragraph{Computational cost.}
For the $\mathcal J$ score, evaluation one candidate requires no
model fitting. It consists of two passes through the data:
\begin{enumerate}
    \item one hash group-by over $T_{\mathrm{tr}}$ to construct
          $\pi_S$, accumulate per-cell label counts, and compute
          $\hat p_S$, $\hat p_0$, and $\Omega$;
    \item one pass over $T_{\mathrm{val}}$ to accumulate the validation
          loss.
\end{enumerate}
Each score evaluation therefore requires
\[
    O(n_{\mathrm{tr}}+n_{\mathrm{val}})
\]
hash operations. At a subset of size $k$, backward elimination evaluates
$k$ neighbours, whereas forward selection evaluates $|F|-k$. A complete
run uses at most $O(|F|^2)$ score evaluations and hence
\[
    O\bigl(
        |F|^2(n_{\mathrm{tr}}+n_{\mathrm{val}})
    \bigr)
\]
work in the worst case. In practice, the cost is usually lower because the
search terminates as soon as no one-column move improves the score by more
than $\tau$.

\paragraph{Summary.}
The complete procedure implements a transparent structural trade-off. Adding
columns may improve validation fidelity by distinguishing more row types, but
it also creates smaller and less-supported cells. The validation loss measures
the first effect, the occupancy penalty measures the second, and
$\mathcal J$ selects the balance between them. The two guarantees justify
these roles separately: \Cref{lem:occ-bound} explains why occupancy controls
estimation, while \Cref{prop:val-oracle} shows that held-out selection is
uniformly reliable over the candidate subsets.

\section{Extended experimental results}

\subsection{Controlled tasks}

\subsubsection{Column recovery under the alignment objective (RQ1)}
\label{app:rq1}

RQ1 asks whether minimizing $\mathcal J$ selects the columns that generate the task and
rejects irrelevant ones. We answer it at two levels: whether the \emph{global minimizer} of $\mathcal J$ is the true relevant set, and whether the \emph{greedy procedure} we run reaches that minimizer. The second matters only because it shows that, despite being greedy, the algorithm performs. 

\subsection{Data-generating process}
\label{app:rq1-dgp-data-generation}
As base data we use the Census/Adult dataset~\cite{adult-dataset}, whose $8$ categorical columns form the core of the column universe $C$. A task is defined by a ground-truth relevant set $S^\circ\subseteq C$, and a label $y$ generated from $S^\circ$. We constrain the base rate to $P(y{=}1)\in[0.3,0.7]$ at generation time to avoid degenerate targets.

The five mechanism families vary in how $y$ depends on $S^\circ$, from row-local
to genuinely relational:
\begin{itemize}
  \item \textbf{Single-value.} $y$ is set by one value of one column:
    $y_i=\mathbbm{1}[T_i[c_a]{=}v]$, with $S^\circ=\{c_a\}$.
  \item \textbf{Conjunction.} $y$ requires two column values jointly:
    $y_i=\mathbbm{1}[T_i[c_a]{=}v_a \wedge T_i[c_b]{=}v_b]$, $S^\circ=\{c_a,c_b\}$, with each column individually informative but neither sufficient.
  \item \textbf{Interaction (XOR).} $y$ is the exclusive-or of two binarized
    columns, chosen so that neither column is informative about $y$ on its own
    while the pair determines it. $S^\circ=\{c_a,c_b\}$. We include this family to stress-test greedy selection.
  \item \textbf{Count.} $y$ depends on how often a row's key value recurs:
    $y_i=\mathbb{1}[\,\text{multiplicity of } T_i[k] \ge \tau\,]$,
    $S^\circ=\{k\}$. No row-local feature on $k$ expresses $y$.
  \item \textbf{Duplicate.} $y$ marks rows whose value in a particular column appears more than once: $y_i=\mathbbm{1}[\,\text{value } T_i[d] \text{ is not unique}\,]$, $S^\circ=\{d\}$.
\end{itemize}

\paragraph{Noise.} We flip each label independently with probability
$\eta\in\{0.0, 0.1, 0.2\}$ to test the robustness and performance of the model against noisy targets.

%\subsection{Objective and selection}
%\label{app:rq1-obj-def}
%A column subset $S$ under signature $\sigma$ induces, through the incidence grable, a partition $\pi_S$ of the rows: two rows share a block iff they share a $\sigma$-signature on $S$ (missing values form their own block). Since rows (nodes) in a block are 1-WL-indistinguishable, no 1-WL-bounded learner can separate them, so the best any such learner can do is predict each block's majority label. The validation risk of that block predictor,

%\[
%  R_{\mathrm{val}}(S)=\tfrac{1}{|V|}\sum_{i\in V}
%  \mathbf{1}\!\left[h_{\pi_S}(i)\ne y_i\right],
%\]
%is therefore the achievable risk ceiling on $\pi_S$, which is what makes $\mathcal J(S)=R_{\mathrm{val}}(S)+\lambda\,\Omega(S)$ a principled score rather than an arbitrary one. 

\subsection{Metrics}
Let $\hat S$ be the selected subset. \emph{Recall}
$\mathrm{Rec}=|\hat S\cap S^\circ|/|S^\circ|$ measure how much of the true set is kept; \emph{Exact}$=\mathbbm{1}[\hat S=S^\circ]$ is their all-or-nothing combination. We also separate objective failure from procedure failure: objective failure is when the minimum of $\mathcal J$ does not correspond to the partition induced by the ground-truth relevant columns, and procedure failure is when the point reached by the greedy procedure is non-optimal. 

\subsection{Protocol}
For each (family, $\eta$) we draw $N=10$ seeds, resampling the rows, planted columns, and label noise each time, and score the selected subset against the known $S^\circ$.
\subsection{Extended results}
\label{app:extended-results}

Beyond the results in Table~\ref{tab:rq1-main}, we report the full
experimental results for the different \autograble mechanisms and for
the inclusion of progressive noise in the tasks. We analyse each axis
separately.

\paragraph{Effect of the regularisation strength $\lambda$.}
Figure~\ref{fig:autograble-lambda} reports performance as a function of
$\lambda$ for every combination of the two \autograble axes---encoding
(\emph{values}, \emph{frequencies}) and direction (\emph{forward},
\emph{backward})---on the clean tasks (noise$=0$). Solid curves show the
mean over the tasks in each category (\emph{row-local},
\emph{extension-sensitive}); faint curves show the individual tasks. The
location of the optimal $\lambda$ shifts across panels and differs between
the two task categories, showing that the effect of $\lambda$ is not
independent of the chosen mechanism.

\begin{figure}[t]
  \centering
  \includegraphics[width=\linewidth]{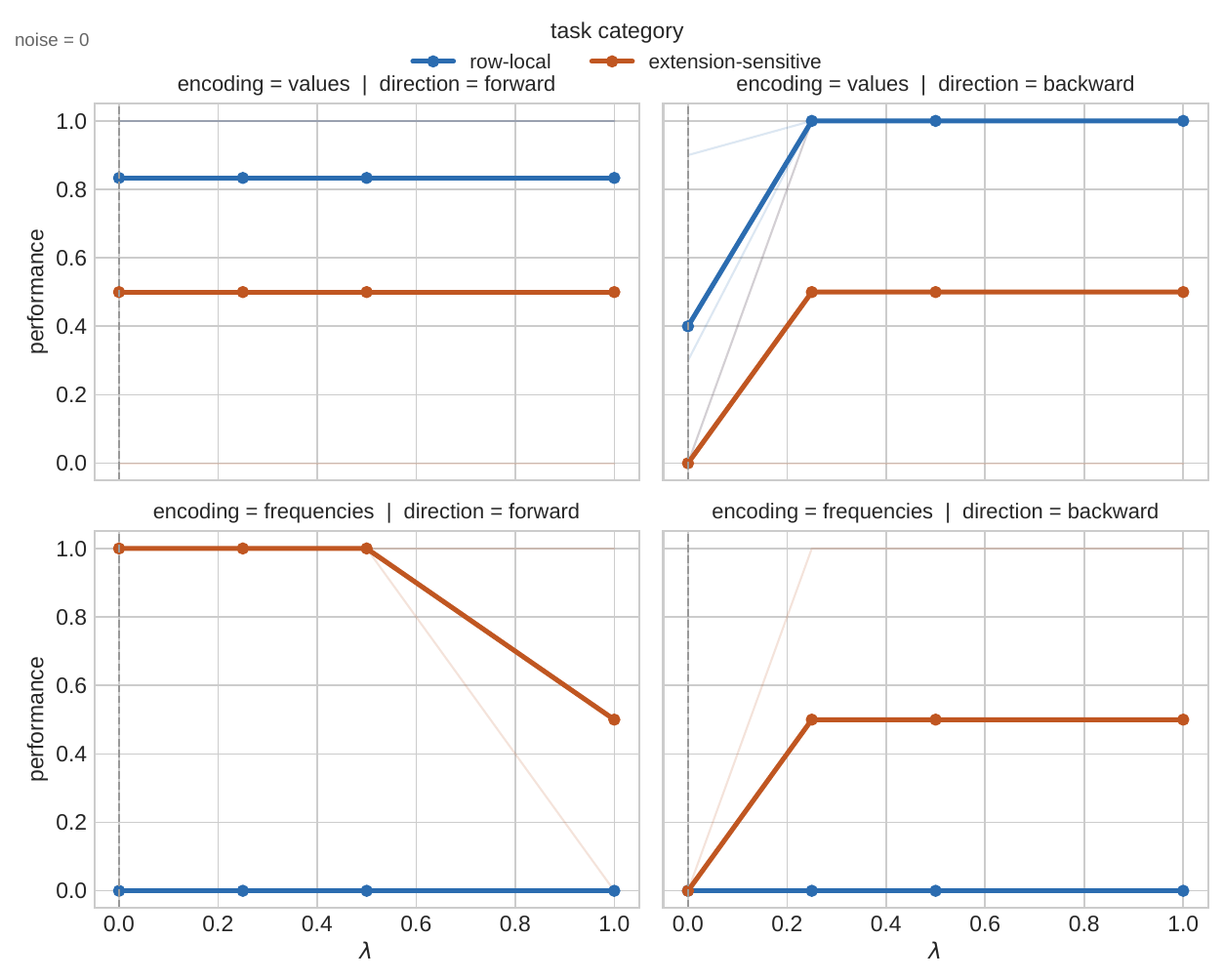}
  \caption{Performance (Exact recovery) versus regularisation strength $\lambda$ for each
    encoding$\,\times\,$direction setting, on the clean tasks
    (noise${}=0$). Solid lines are per-category means over tasks; faint
    lines are individual tasks. Axes are shared across panels. Best viewed in color. }
  \label{fig:autograble-lambda}
\end{figure}

\paragraph{Robustness to progressive noise.}
Figure~\ref{fig:autograble-noise} reports performance as the amount of
added noise increases, shown separately for each task category
(Figures~\ref{fig:autograble-noise-rowlocal}
and~\ref{fig:autograble-noise-extsens}). Within each panel---again one per
encoding$\,\times\,$direction setting---each curve corresponds to a value
of $\lambda$. This lets us read off both the overall degradation under
noise and whether the best $\lambda$ depends on the noise level: crossing
curves indicate that the ranking of $\lambda$ changes as noise grows.
% TODO: state whether the best lambda is stable or changes under noise.

\begin{figure}[t]
  \centering
  \begin{subfigure}{\linewidth}
    \centering
    \includegraphics[width=0.8\linewidth]{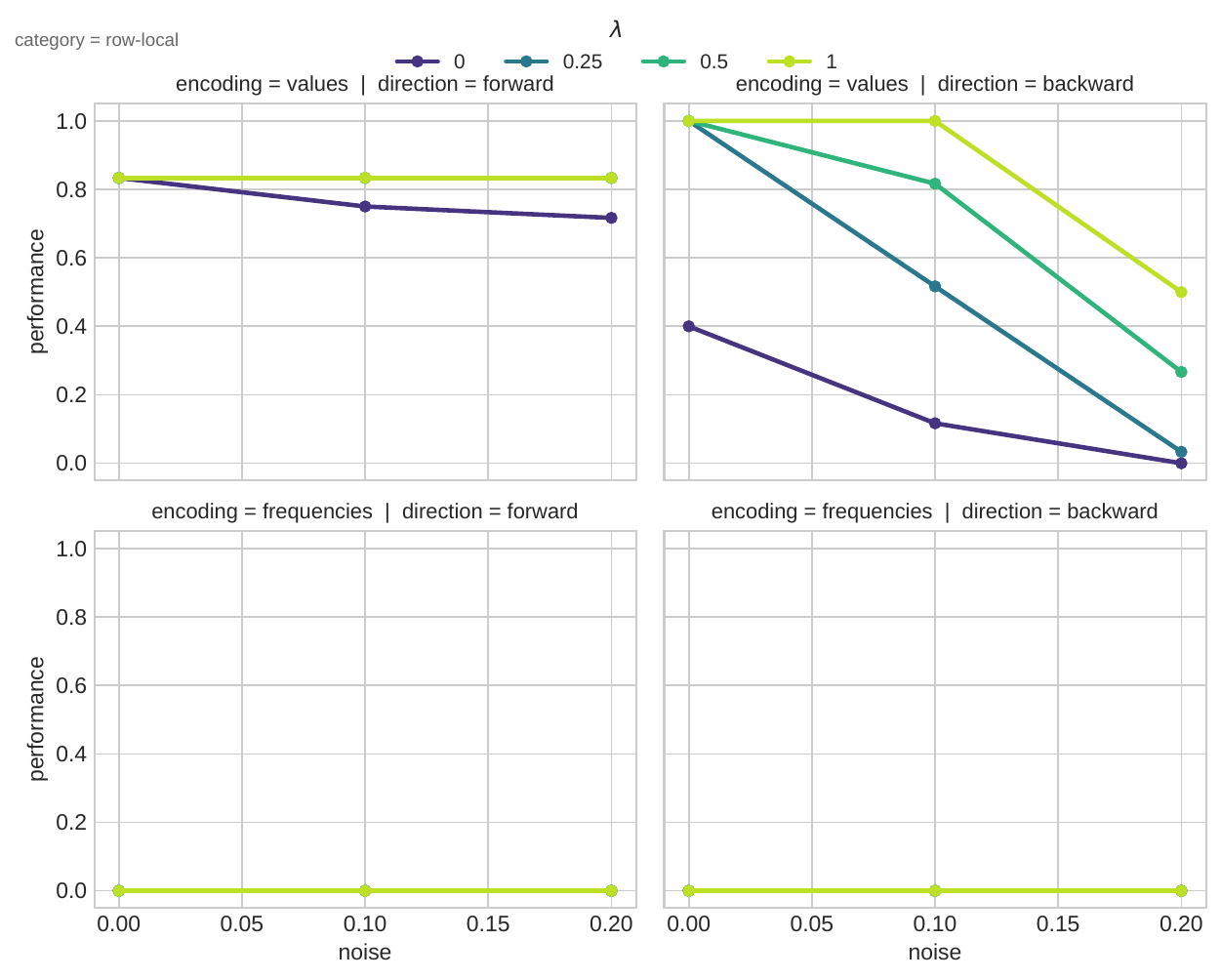}
    \caption{\emph{Row-local} tasks.}
    \label{fig:autograble-noise-rowlocal}
  \end{subfigure}

  \vspace{0.3em}

  \begin{subfigure}{\linewidth}
    \centering
    \includegraphics[width=0.8\linewidth]{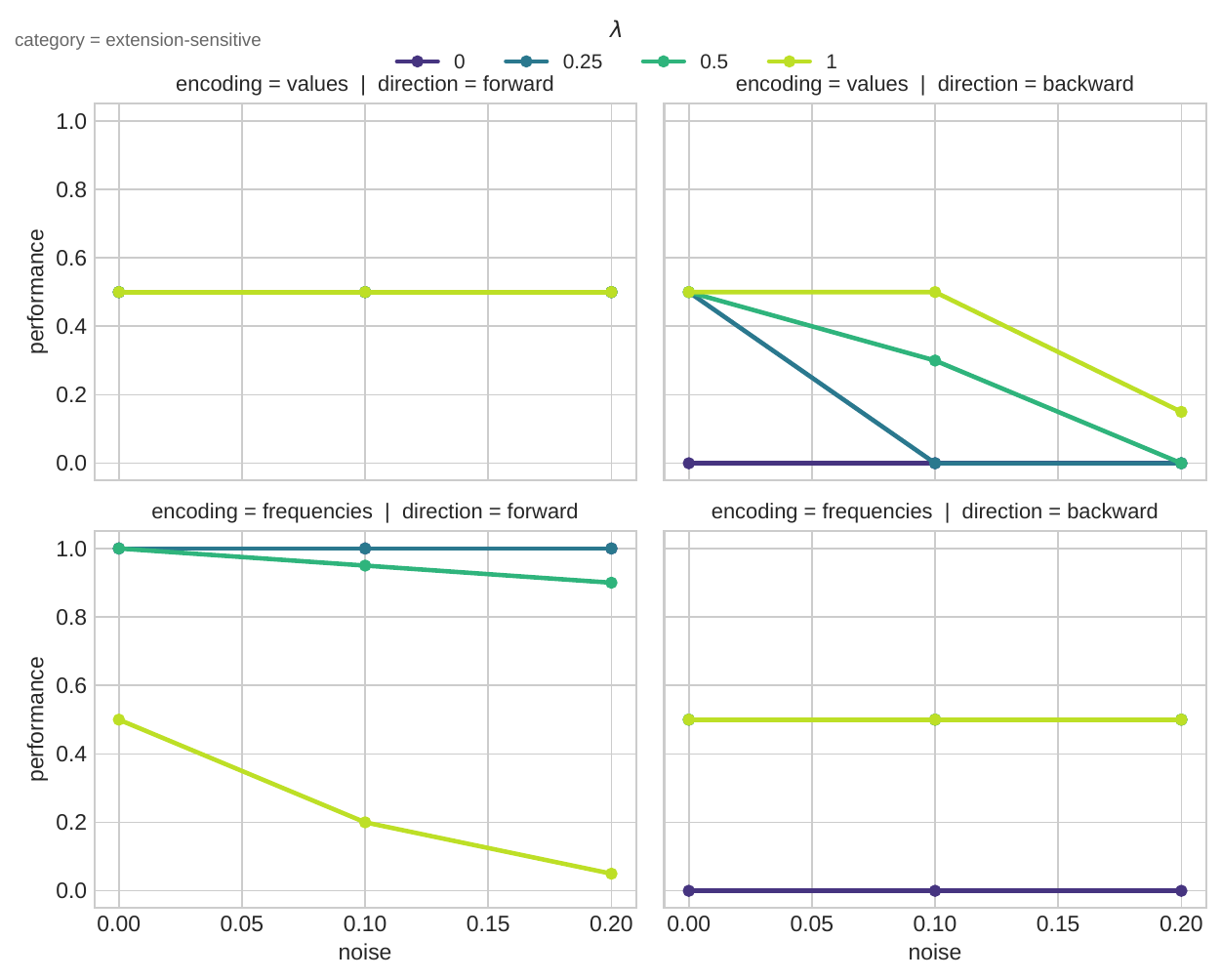}
    \caption{\emph{Extension-sensitive} tasks.}
    \label{fig:autograble-noise-extsens}
  \end{subfigure}
  \caption{Performance (Exact recovery) under progressively stronger added noise, per task
    category. Each panel corresponds to an encoding$\,\times\,$direction
    setting; each curve corresponds to a value of $\lambda$ (light to dark). Best viewed in color. }
  \label{fig:autograble-noise}
\end{figure}

\subsection{Real transactional, relational and i.i.d tasks}
\label{app:real-extended-results}

\subsubsection{Training details}
\label{appendix:details-hyperparameters-real-tasks}

\paragraph{GraphSAGE architecture and hyperparameters} Throughout the transactional, relational and i.i.d.\ tabular tasks, the predictor is kept fixed: for each dataset, a single configuration is instantiated once and reused unchanged across every task family and every graph construction evaluated on that dataset. Table~\ref{tab:sage-hparams-fraud-tabular} reports the architecture and configuration used for the transactional (fraud) and tabular i.i.d.\ (TabArena) setups, which share the same \texttt{HeteroSAGE} implementation. For RelBench tasks we use RelBench's own reference \texttt{HeteroGraphSAGE} implementation and hyperparameters, reported separately in Table~\ref{tab:sage-hparams-relational}, since its encoder and message-passing modules differ from ours and are not directly comparable field-by-field.

Hyperparameters are not selected per construction: the validation set is used only for early stopping, checkpointing on the epoch with the best validation score, under a fixed architecture and fixed hyperparameters, not for a hyperparameter search re-run per construction. No construction receives more tuning than another. 

\begin{table}[t]
\centering
\caption{GraphSAGE (\texttt{HeteroSAGE}) hyperparameters, Transactional (fraud) vs.\ Tabular IID (TabArena) setups. Both use the same architecture (\texttt{NodeEncoder} + \texttt{HeteroSAGE} + \texttt{MLPHead}).}
\label{tab:sage-hparams-fraud-tabular}
\begin{tabular}{lcc}
\toprule
Hyperparameter & Transactional & Tabular IID \\
\midrule
Hidden dimension              & 256 & 256 \\
Layers                        & 4 & 4 \\
Head layers                   & 2   & 2   \\
Dropout                       & 0.15 & 0.15 \\
Message aggregation (within relation) & sum   & sum   \\
Relation aggregation (across relations) & cat & cat \\
Residual                      & True   & True   \\
Normalization                  & layer   & layer   \\
Learning rate                  & 2e-4 & 2e-4 \\
Weight decay                   & 1e-5   & 1e-5   \\
Epochs (max)                   & 200 & 200 \\
Patience                       & 20 & 20 \\
Gradient clip                  & 1.0   & 1.0   \\
Selection metric                & AP   & AUC \\
Optimizer                       & AdamW & AdamW \\
\bottomrule
\end{tabular}
\end{table}

\begin{table}[t]
\centering
\caption{GraphSAGE hyperparameters, Relational (RelBench) setup. Uses
RelBench's \texttt{HeteroEncoder} + \texttt{HeteroTemporalEncoder} +
\texttt{HeteroGraphSAGE} + single-layer \texttt{MLP} head}
\label{tab:sage-hparams-relational}
\begin{tabular}{lc}
\toprule
Hyperparameter & Relational \\
\midrule
Hidden dimension / channels    & 128 \\
Layers                          & 4 \\
Head layers                     & 1 \\
Message aggregation              & sum \\
Normalization                     & batch\_norm \\
Learning rate                      & 5e-3 \\
Epochs (max)                         & 200 \\
Patience / early stopping             & 10\\
Temporal encoder                           &\texttt{HeteroTemporalEncoder} \\
Optimizer                                   & Adam \\
\bottomrule
\end{tabular}
\end{table}

\subsubsection{Transactional fraud detection}
From the \emph{FDB: Fraud Dataset Benchmark} we use all available real datasets with at least 3 categorical attributes. This rules out 5 of the 9 available datasets, from which 4 have two or less categorical variables (and hence the graph-construction-via-selection effect is limited) or are simulated (particularly ``Simulated Credit Card Transactions generated using Sparkov'' (\texttt{sparknov})). Table \ref{tab:results-transactions-full} show the validation and test results of the fixed graph constructions and the different \autograble configurations. 

% Requires: \usepackage{booktabs, graphicx, multirow}
\begin{table*}[htbp]
\centering
\caption{\autograble performance (AUC) comparison across encoding/direction/$\lambda$ configurations on transactional (FDB) datasets. Results averaged over $N{=}15$ seeds. We note that \texttt{fraudecom} is a particularly challenging task, where no configuration (nor other methods reported in the official benchmark leader board) beats near-random performance.}
\label{tab:results-transactions-full}
\resizebox{0.85\textwidth}{!}{%
\begin{tabular}{ll c cc cc cc cc}
\toprule
\multirow{2}{*}{Encoding} & \multirow{2}{*}{Direction} & \multirow{2}{*}{$\lambda$}
 & \multicolumn{2}{c}{\texttt{fraudecom}} & \multicolumn{2}{c}{\texttt{twitterbot}} & \multicolumn{2}{c}{\texttt{vehicleloan}} & \multicolumn{2}{c}{\texttt{fakejob}} \\
\cmidrule(lr){4-5}\cmidrule(lr){6-7}\cmidrule(lr){8-9}\cmidrule(lr){10-11}
 & & & Val & Test & Val & Test & Val & Test & Val & Test \\
\midrule
\multicolumn{3}{l}{Trivial ($\gamma_{\mathrm{triv}}$)}
 & 0.509 &0.499& 0.866 & 0.864 & 0.646 & 0.647 & 0.925 & 0.911 \\
\multicolumn{3}{l}{Full incidence ($\gamma_{\mathrm{inc}}$)}
 & 0.511 & 0.501 & 0.821 & 0.831 & 0.630 & 0.619& \bf 0.938 & \bf 0.951 \\
\midrule
Values    & Backward & 0 & 0.511          & 0.492 & 0.837          & 0.831 & 0.647          & 0.642 & 0.930     & 0.929 \\
Values    & Backward & 1 & 0.502          & 0.507 & 0.837          & 0.831 & 0.649          & 0.645 & 0.923     & 0.910 \\
Values    & Forward  & 0 & 0.525          & 0.516 & \textbf{0.916} & 0.919 & 0.653          & 0.649 & 0.917     & 0.897 \\
Values    & Forward  & 1 & 0.503          & 0.497 & 0.914          & 0.917 & 0.646          & 0.645 & 0.923     & 0.910 \\
\midrule
Frequency & Backward & 0 & 0.525          & 0.513 & 0.837          & 0.831 & 0.642          & 0.644 & \bf 0.938 & \bf 0.951 \\
Frequency & Backward & 1 & 0.525          & 0.513 & 0.912          & 0.911 & 0.646          & 0.647 & 0.921     & 0.907 \\
Frequency & Forward  & 0 & \textbf{0.527} & 0.513 & 0.893          & 0.880 & 0.646          & 0.647 & 0.922     & 0.907 \\
Frequency & Forward  & 1 & \textbf{0.527} & 0.513 & 0.913          & 0.908 & \textbf{0.676} & 0.662 & 0.925     & 0.911 \\
\bottomrule
\end{tabular}}
\end{table*}

\subsubsection{Tabular datasets results}
\label{app:tabarena}
Table~\ref{tab:results-tabarena-full} reports the performance of $\gamma_{\mathrm{triv}}$, $\gamma_{\mathrm{inc}}$, and the different
configurations of \autograble on a selected subset of the i.i.d.\ tabular benchmark TabArena~\cite{tabarena}. The datasets are selected for having more than 8 categorical features and a binary classification target. 

% Requires: \usepackage{booktabs, graphicx, multirow}

% Requires: \usepackage{booktabs, graphicx, multirow}
\begin{table*}[htbp]
\centering
\caption{TabArena (i.i.d.\ single-table) results. AUC averaged over $N{=}15$ seeds.
Baselines above the rule; \autograble\ configurations below, with bold marking the
best validation score among them. A $^{*}$ marks a configuration that selects no
columns and therefore returns $\gamma_{\mathrm{triv}}$; a $^{\circ}$ marks one that
selects all of them and therefore returns $\gamma_{\mathrm{inc}}$. Since each label
depends on its own row, no constructor is expected to improve on
$\gamma_{\mathrm{triv}}$.}
\label{tab:results-tabarena-full}
\resizebox{\textwidth}{!}{%
\begin{tabular}{ll c cc cc cc cc cc}
\toprule
\multirow{2}{*}{Encoding} & \multirow{2}{*}{Direction} & \multirow{2}{*}{$\lambda$}
 & \multicolumn{2}{c}{\texttt{credit-g}} & \multicolumn{2}{c}{\texttt{good-customer}} & \multicolumn{2}{c}{\texttt{mkt\_campaign}} & \multicolumn{2}{c}{\texttt{NATICUSdroid}} & \multicolumn{2}{c}{\texttt{qsar-biodeg}} \\
\cmidrule(lr){4-5}\cmidrule(lr){6-7}\cmidrule(lr){8-9}\cmidrule(lr){10-11}\cmidrule(lr){12-13}
 & & & Val & Test & Val & Test & Val & Test & Val & Test & Val & Test \\
\midrule
\multicolumn{3}{l}{Trivial ($\gamma_{\mathrm{triv}}$)}
 & 0.769 & 0.764 & 0.683 & 0.762 & 0.908 & 0.886 & 0.984 & 0.983 & 0.688 & 0.715 \\
\multicolumn{3}{l}{Full incidence ($\gamma_{\mathrm{inc}}$)}
 & 0.800 & 0.720 & 0.694 & 0.740 & 0.839 & 0.776 & 0.985 & 0.982 & 0.693 & 0.724 \\
\midrule
Values    & Backward & 0 & \bf 0.807 & 0.739 & 0.692 & 0.756 & 0.917 & 0.886 & 0.984 & 0.983 & 0.689 & 0.724 \\
Values    & Backward & 1 & \bf 0.807 & 0.739 & 0.692 & 0.756 & \bf 0.918 & 0.888 & 0.984 & 0.983 & 0.696 & 0.714 \\
Values    & Forward  & 0 & 0.767 & 0.774 & 0.683$^{*}$ & 0.762 & 0.903 & 0.887 & 0.984 & 0.984 & 0.693$^{\circ}$ & 0.724 \\
Values    & Forward  & 1 & 0.769 & 0.764 & 0.683$^{*}$ & 0.762 & 0.909 & 0.886 & \bf 0.985 & 0.984 & 0.693$^{\circ}$ & 0.724 \\
\midrule
Frequency & Backward & 0 & \bf 0.807 & 0.739 & 0.694$^{\circ}$ & 0.740 & 0.913 & 0.880 & 0.983 & 0.982 & 0.693 & 0.725 \\
Frequency & Backward & 1 & \bf 0.807 & 0.739 & \bf 0.699 & 0.766 & 0.909 & 0.883 & 0.983 & 0.986 & 0.693 & 0.715 \\
Frequency & Forward  & 0 & 0.769$^{*}$ & 0.764 & 0.694 & 0.737 & 0.908$^{*}$ & 0.886 & 0.984$^{*}$ & 0.985 & 0.688$^{*}$ & 0.715 \\
Frequency & Forward  & 1 & 0.769$^{*}$ & 0.764 & 0.683$^{*}$ & 0.762 & 0.908$^{*}$ & 0.886 & 0.984$^{*}$ & 0.985 & \bf 0.705 & 0.714 \\
\bottomrule
\end{tabular}}
\end{table*}

Here the labels are allegedly i.i.d.\ at the row level, so the informative outcome should be that no structure is selected. In every task at least one \autograble configuration returns the empty column set and so reduces to $\gamma_{\mathrm{triv}}$, and in four of five at least two do. All ten such cases are forward search, and eight of them are frequency-encoded forward search: starting from the empty set, no column offers a validation gain that covers the fragmentation it adds. Building structure
anyway is not neutral. Full incidence gains on validation and loses on test
relative to $\gamma_{\mathrm{triv}}$ ($0.800\!\to\!0.720$ on \texttt{credit-g},
$0.839\!\to\!0.776$ on \texttt{mkt\_campaign}, against drops of $0.005$ and $0.022$), which is the fragmentation-driven overfitting that $\Omega$ is meant to detect.

\subsubsection{Relational databases results}

Table \ref{tab:full-relbench} %and Table \ref{tab:full-relbench-more}
shows results of \autograble and the other baselines in both validation and test set of Relbench tasks. 

% Requires: \usepackage{booktabs, graphicx}  % graphicx for \resizebox
% Optional: \usepackage{adjustbox} for finer width control

\begin{table}[t]
\centering
\caption{Validation and test AUROC ($\uparrow$) under a \emph{fixed} predictor; only the
constructor varies. Mean\,$\pm$\,std over $N{=}15$ seeds. Best in \textbf{bold},
second \underline{underlined}.}
\label{tab:full-relbench}
\resizebox{\textwidth}{!}{%
\begin{tabular}{l cc cc cc}
\toprule
 & \multicolumn{2}{c}{\texttt{driver-top3} (RelBench-F1)} & \multicolumn{2}{c}{\texttt{driver-dnf} (RelBench-F1)} & \multicolumn{2}{c}{\texttt{study-outcome} (RelBench-Trial)} \\
\cmidrule(lr){2-3}\cmidrule(lr){4-5}\cmidrule(lr){6-7}
Constructor & Val & Test & Val & Test & Val & Test \\
\midrule
Trivial ($\gamma_{\mathrm{triv}}$)
  & \pms{0.640}{0.017} & \pms{0.674}{0.011}
  & \pms{0.610}{0.020} & \pms{0.596}{0.026}
  & \pmsB{0.648}{0.019} & \pmsB{0.677}{0.009} \\
REG
  & \pms{0.803}{0.009} & \pms{0.777}{0.012}
  & \pms{0.737}{0.007} & \pms{0.733}{0.013}
  & \pms{0.633}{0.011} & \pms{0.635}{0.006} \\
Full incidence ($\gamma_{\mathrm{inc}}$)
  & \pms{0.762}{0.020} & \pms{0.746}{0.013}
  & \pms{0.698}{0.017} & \pms{0.687}{0.022}
  & \pms{0.628}{0.031} & \pms{0.621}{0.013} \\
Random
  & \pms{0.767}{0.030} & \pms{0.759}{0.025}
  & \pms{0.741}{0.015} & \pms{0.733}{0.012}
  & \pms{0.598}{0.021} & \pms{0.602}{0.022} \\
auGraph
  & \pmsU{0.810}{0.019} & \pmsU{0.791}{0.009}
  & \pms{0.750}{0.014} & \pmsU{0.742}{0.021}
  & \pms{0.617}{0.018} & \pms{0.630}{0.017} \\
\midrule
\autograble
  & \pmsB{0.845}{0.013} & \pmsB{0.803}{0.011}
  & \pmsB{0.770}{0.012} & \pmsB{0.761}{0.021}
  & \pmsB{0.648}{0.019} & \pmsB{0.677}{0.009} \\
%\autograble-forward-values & \pms{0.00}{0.00} & \pms{0.00}{0.00} & \pms{0.00}{0.00} & \pms{0.00}{0.00} & \pms{0.00}{0.00} & \pms{0.00}{0.00} \\
%\autograble-forward-frequencies & \pms{0.00}{0.00} & \pms{0.00}{0.00} & \pms{0.00}{0.00} & \pms{0.00}{0.00} & \pms{0.00}{0.00} & \pms{0.00}{0.00} \\
%\autograble-backward-values & \pms{0.00}{0.00} & \pms{0.00}{0.00} & \pms{0.00}{0.00} & \pms{0.00}{0.00} & \pms{0.00}{0.00} & \pms{0.00}{0.00} \\
%\autograble-backward-frequencies & \pms{0.00}{0.00} & \pms{0.00}{0.00} & \pms{0.00}{0.00} & \pms{0.00}{0.00} & \pms{0.00}{0.00} & \pms{0.00}{0.00} \\
\bottomrule
\end{tabular}}
\end{table}

\end{document}